\documentclass{svjour3}
\journalname{Machine learning}

\usepackage[bitstream-charter]{mathdesign}
\usepackage{amsmath}
\usepackage{inconsolata}
\usepackage{pgfplots}
\usepackage{algorithm}
\usepackage{listings}
\usepackage{caption}
\usepackage{subcaption}
\usepackage{booktabs}
\usepackage[normalem]{ulem}

\smartqed
\let\oldendproof\endproof
\renewcommand{\endproof}{\qed\oldendproof}

\newcommand{\metagol}{Metagol}
\newcommand{\metaho}{Metagol$_{ho}$}
\newcommand{\namea}{Metagol$_{ho}$}
\newcommand{\nameb}{HEXMIL$_{ho}$}
\newcommand{\M}[2]{$\mathcal{M}^{#1}_{#2}$}
\newcommand{\tw}[1]{\texttt{#1}}

\newcommand{\hex}{HEXMIL}
\newcommand{\hexho}{\hex$_{ho}$}

\usepackage{xcolor}

\title{Learning higher-order logic programs}
\author{Andrew Cropper \and Rolf Morel \and Stephen Muggleton}

\institute{A. Cropper\at
              University of Oxford, UK\\
              \email{andrew.cropper@cs.ox.ac.uk}
            \and
            R. Morel\at
              University of Oxford, UK\\
              \email{rolf.morel@cs.ox.ac.uk}
            \and
                S. H. Muggleton \at
                Imperial College London, UK\\
              \email{s.muggleton@imperial.ac.uk}
}

\begin{document}

\maketitle

\begin{abstract}
A key feature of inductive logic programming (ILP) is its ability to learn first-order programs, which are intrinsically more expressive than propositional programs.
In this paper, we introduce techniques to learn higher-order programs.
Specifically, we extend meta-interpretive learning (MIL) to support learning higher-order programs by allowing for \emph{higher-order definitions} to be used as background knowledge.
Our theoretical results show that learning higher-order programs, rather than first-order programs, can reduce the textual complexity required to express programs which in turn reduces the size of the hypothesis space and sample complexity.
We implement our idea in two new MIL systems: the Prolog system \namea{} and the ASP system \nameb{}.
Both systems support learning higher-order programs and higher-order predicate invention, such as inventing functions for \tw{map/3} and conditions for \tw{filter/3}.
We conduct experiments on four domains (robot strategies, chess playing, list transformations, and string decryption) that compare learning first-order and higher-order programs.
Our experimental results support our theoretical claims and show that, compared to learning first-order programs, learning higher-order programs can significantly improve predictive accuracies and reduce learning times.
\end{abstract}
\section{Introduction}
\label{sec:intro}

Suppose you have intercepted encrypted messages and you want to learn a general decryption program from them.
Figure \ref{fig:introexs} shows such a scenario with three example encrypted/decrypted strings.
In this scenario the underlying encryption algorithm is a simple Caesar cipher with a shift of +1.
Given these examples, most inductive logic programming (ILP) approaches, such as meta-interpretive learning (MIL) \cite{mugg:metalearn,mugg:metagold}, would learn a recursive first-order program, such as the one shown in Figure \ref{fig:introfo}.
Although correct, this first-order program is overly complex in that most of the program is concerned with manipulating the input and output, such as getting the head and tail elements.
In this paper, we introduce techniques to learn higher-order programs that abstract away this boilerplate code.
Specifically, we extend MIL to support learning higher-order programs that use higher-order constructs such as \tw{map/3}, \tw{until/4}, and \tw{ifthenelse/5}.
Using this new approach, we can learn an equivalent\footnote{Success set equivalent when restricted to the target predicate \tw{decrypt/2}.} yet smaller decryption program, such as the one shown in Figure \ref{fig:introho}, which uses \tw{map/3} to abstract away the recursion and list manipulation.

\begin{figure}[ht]
\centering
\normalsize
\begin{tabular}{l|l}
\textbf{Encrypted} & \textbf{Decrypted}\\
\hline
\tw{joevdujwf} & \tw{inductive}\\
\tw{mphjd} & \tw{logic}\\
\tw{qsphsbnnjoh} & \tw{programming}
\end{tabular}
\caption{Example encrypted and decrypted messages.}
\label{fig:introexs}
\end{figure}

\begin{figure}[ht]
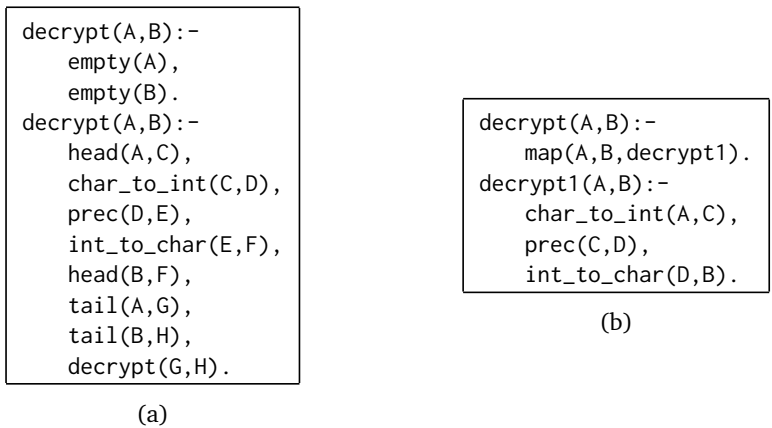

\normalsize
\begin{subfigure}[t]{.5\textwidth}
\centering
\normalsize
\begin{tabular}{|c|}
\hline
\begin{lstlisting}
decrypt(A,B):-
    empty(A),
    empty(B).
decrypt(A,B):-
    head(A,C),
    char_to_int(C,D),
    prec(D,E),
    int_to_char(E,F),
    head(B,F),
    tail(A,G),
    tail(B,H),
    decrypt(G,H).
\end{lstlisting} \\
\hline
\end{tabular}
\normalsize
\caption{}
\label{fig:introfo}
\end{subfigure}%
\begin{subfigure}[t]{.5\textwidth}
\centering
\begin{tabular}{|c|}
\hline
\begin{lstlisting}
decrypt(A,B):-
    map(A,B,decrypt1).
decrypt1(A,B):-
    char_to_int(A,C),
    prec(C,D),
    int_to_char(D,B).
\end{lstlisting} \\
\hline
\end{tabular}
\caption{}
\label{fig:introho}
\end{subfigure}
\caption{Decryption programs.
Figure (a) shows a first-order program.
Figure (b) shows a higher-order program, where \tw{decrypt1/2} is an invented predicate symbol.
The predicate \tw{prec/2} represents \emph{preceding/2}, i.e. the inverse of \emph{successor/2}.
The programs are success set equivalent when restricted to the target predicate \tw{decrypt/2} but the higher-order program is much smaller and requires half the number of literals (6 vs 12).
}
\label{fig:intro}
\end{figure}

\noindent
We claim that, compared to learning first-order programs, learning higher-order programs can improve learning performance.
We support our claim by showing that learning higher-order programs can reduce the textual complexity required to express programs which in turn reduces the size of the hypothesis space and sample complexity.

We implement our idea in \metaho{}, which extends Metagol \cite{metagol}, a MIL implementation based on a Prolog meta-interpreter.
\metaho{} extends Metagol to support \emph{interpreted} BK (IBK).
In this approach, meta-interpretation drives both the search for a hypothesis and predicate invention, allowing for higher-order arguments to be invented, such as the predicate \tw{decrypt1/2} in Figure \ref{fig:introho}.
The key novelty of \metaho{} is the combination of \emph{abstraction} (learning higher-order programs) and \emph{invention} (predicate invention), i.e. inventions inside of abstractions.
\metaho supports the invention of conditions and functions to an arbitrary depth, which goes beyond anything in the literature.
We also introduce \hexho{}, which likewise extends \hex{} \cite{hexmil}, an answer set programming (ASP) MIL implementation, to support learning higher-order programs.
As far as we are aware, \hexho{} is the first ASP-based ILP system that has been demonstrated capable of learning higher-order programs.

We further support our claim that learning higher-order programs can improve learning performance by conducting experiments in four domains: robot strategies, chess playing, list transformations, and string decryption.
The experiments compare the predictive accuracies and learning times when learning first and higher-order programs.
In all cases learning higher-order programs leads to substantial increases in predictive accuracies and lower learning times in agreement with our theoretical results.

Our main contributions are:
\begin{itemize}
\item We extend the MIL framework to support learning higher-order programs by extending it to support higher-order definitions (Section \ref{sec:amil}).
\item We show that the new higher-order approach can reduce the textual complexity of programs which in turn reduces the size of the hypothesis space and also sample complexity (Section \ref{sec:sc}).
\item We introduce \metaho{} and \hexho{} which extend Metagol and \hex{} respectively. Both systems support learning higher-order programs with higher-order predicate invention (Section \ref{sec:algos}).
\item We show that the ASP-based \hex{} and \hexho{} have an additional factor determining the size of their search space, namely the number of constants (Section \ref{sec:complex}).
\item We conduct experiments in four domains which show that, compared to learning first-order programs, learning higher-order programs can substantially improve predictive accuracies and reduce learning times (Section \ref{sec:experiments}).
\end{itemize}


\section{Related work}
\label{related}

\subsection{Program induction}
Program synthesis is the automatic generation of a computer program from a specification.
Deductive approaches \cite{mana:synthesis} \emph{deduce} a program from a full specification which precisely states the requirements and behaviour of the desired program.
By contrast, program induction approaches \emph{induce} (learn) a program from an incomplete specification, usually input/output examples.
Many program induction approaches learn specific classes of programs, such as string transformations \cite{flashfill}.
By contrast, MIL is general-purpose, shown capable of grammar induction \cite{mugg:metalearn}, learning robot strategies \cite{crop:metagolo}, and learning efficient algorithms \cite{crop:metaopt}.
In addition, MIL supports \emph{predicate invention}, which has been repeatedly stated as an important challenge in ILP \cite{mugg:cigol,stahl:pi,mlj:ilp20}.
The idea behind predicate invention is for an ILP system to introduce new predicate symbols to improve learning performance. In program induction, predicate invention can be seen as inventing auxiliary functions/predicates, as one does when manually writing a program, for example to reduce code duplication or to improve the readability of a program.

\subsection{Inductive functional programming}
Functional program induction approaches often support learning higher-order programs.
MagicHaskeller \cite{magichaskeller} is a general-purpose system which learns Haskell functions by selecting and instantiating higher-order functions from a pre-defined vocabulary.
Igor2 \cite{igor2} also learns recursive Haskell programs and supports auxiliary function invention but is restricted in that it requires the first $k$ examples of a target theory to generalise over a whole class.
The \emph{L2} system \cite{feser:pldi15} synthesises recursive functional algorithms.
The \emph{MYTH} \cite{OseraZ15} and \emph{MYTH2} \cite{FrankleOWZ16} systems use type systems to synthesise programs.
Frankle et al.~\cite{FrankleOWZ16} show how example-based specifications can be turned into type specifications.
In this work we go beyond these approaches by (1) learning higher-order programs with invented predicates, (2) giving theoretical justifications and conditions for when learning higher-order programs can improve learning performance (Section \ref{sec:sc}), and (3) experimentally demonstrating that learning higher-order programs can improve learning performance.

\subsection{Inductive logic programming}
ILP systems, including the popular systems FOIL \cite{foil}, Progol \cite{mugg:progol}, ALEPH \cite{aleph}, and TILDE \cite{tilde}, usually learn first-order programs.
Given appropriate mode declarations \cite{mugg:progol} for higher-order predicates such as \tw{map/3}, Progol and Aleph could learn higher-order programs such as \tw{f(A,B):-map(A,B,f1)}.
However, because Progol and Aleph do not support predicate invention they would be unable to invent the predicate \tw{f1/2} in the above example.
Similarly, existing MIL implementations, such as Metagol, could learn a similar program to the one above when \tw{map/3} is provided as background knowledge.
However, even though Metagol supports predicate invention, it is unable to invent the predicate \tw{f1/2} in the example above because Metagol deductively proves BK by delegating the proofs to Prolog.
To overcome this limitation we introduce the notion of interpreted BK (IBK), where \tw{map/3} can be defined as IBK.
The new MIL system \metaho{} proves IBK through meta-interpretation, which allows for predicate arguments such as \tw{f1/2} to be invented.

\subsection{Meta-interpretive learning}
MIL was originally based on a Prolog meta-interpreter, although the MIL problem has also been encoded as an ASP problem \cite{hexmil}.
The key difference between a MIL learner and a standard Prolog meta-interpreter is that whereas a standard Prolog meta-interpreter attempts to prove a goal by repeatedly fetching first-order clauses whose heads unify with a given goal, a MIL learner additionally attempts to prove a goal by fetching higher-order existentially quantified formulas, called \emph{metarules}, supplied as BK, whose heads unify with the goal.
The resulting predicate substitutions are saved and can be reused later in the proof.
Following the proof of a set of goals, a logic program is induced by projecting the predicate substitutions onto their corresponding metarules.
A key feature of MIL is the support for predicate invention.
MIL uses predicate invention for automatic problem decomposition.
As we will demonstrate, the combination of predicate invention and abstraction leads to compact representations of complex programs.


Cropper and Muggleton \cite{crop:metafunc} introduced the idea of using MIL to learn higher-order programs by using IBK.
This paper is an extended version of that paper.
In addition, we go beyond that work in several ways.
First, we generalise their preliminary theoretical results, principally in Section \ref{sec:sc}.
We also provide more explanation as to why abstracted MIL can improve learning performance compared to unabstracted MIL (end of Section \ref{sec:sc}).
Second, we introduce the \hexho{} system, which, as mentioned, extends HEXMIL to support learning higher-order programs with higher-order predicate invention.
Our motivation for this extension is to show the generality of our work, i.e. to demonstrate that it is not specific to Metagol and Prolog.
We also study the computational complexity of both \metaho{} and \hexho{}.
We show that the ASP approach is highly sensitive to the number of constant symbols, which leads to scalability issues.
Furthermore, we corroborate the experimental results of Cropper and Muggleton by repeating the robot waiter, chess, and list transformation experiments with \metaho{}.
We provide additional experimental evidence by repeating the experiments with \hexho{}.
Finally, we add further evidence by conducting a new experiment on the string decryption problem mentioned in the introduction.

\subsection{Higher-order logic}
McCarthy \cite{DBLP:conf/mi/McCarthy95} advocated using higher-order logic to represent knowledge.
Similarly, Muggleton et al.~\cite{mlj:ilp20} argued that using higher-order representations in ILP provides more flexible ways of representing BK.
Lloyd \cite{lloyd:logiclearning} used higher-order logic in the learning process but the approach focused on learning functional programs and did not support predicate invention.
Early work in ILP \cite{flener99,raedt:clint,emde:metarules} used higher-order formulae to specify the overall form of programs to be learned, similar to how MIL uses metarules.
However, these works did not consider learning higher-order programs.
By contrast, we use higher-order logic as a learning representation and to represent learned hypotheses.
Feng and Muggleton \cite{mugg:hol-inductive-generalisation} investigated inductive generalisation in higher-order logic using a restricted form of lambda calculus.
However, their approach does not support first-order nor higher-order predicate invention.
By contrast, we introduce higher-order definitions which allow for invented predicate symbols to be used as arguments in literals.

\subsection{Abstraction and invention}
Predicate invention has been repeatedly stated as an important challenge in ILP \cite{mugg:cigol,stahl:pi,mlj:ilp20}.
Popular ILP systems, such as FOIL, Progol, and ALEPH, do not support predicate invention, nor do most program induction systems.
Meta-level abduction \cite{inoue:mla} uses abduction and meta-level reasoning to invent predicates that represent propositions.
By contrast, MIL uses abduction to invent predicates that represent relations, i.e. relations that are not in the initial BK nor in the examples.
For instance, MIL was shown \cite{mugg:metagold} able to invent a predicate corresponding to the \emph{parent/2} relation when learning a \emph{grandparent/2} relation.
In this paper we extend MIL and the associated Metagol implementation to support higher-order predicate invention for use in higher-order constructs, such as \tw{map/3}, \tw{reduce/3}, and \tw{fold/4}.
This approach supports a form of abstraction which goes beyond typical first-order predicate invention \cite{saitta2013abstraction} in that the use of higher-order definitions combined with meta-interpretation drives both the search for a hypothesis and predicate invention, leading to more accurate and compact programs.
\section{Theoretical framework}
\label{framework}

\subsection{Preliminaries}

We assume familiarity with logic programming.
However, we restate key terminology.
Note that we focus on learning function-free logic programs, so we ignore terminology to do with function symbols.
We denote the predicate and constant signatures as $\mathcal{P}$ and $\mathcal{C}$ respectively.
A variable is first-order if it can be bound to a constant symbol or another first-order variable.
A variable is higher-order if it can be bound to a predicate symbol or another higher-order variable.
We denote the sets of first-order and higher-order variables as $\mathcal{V}_1$ and $\mathcal{V}_2$ respectively.
A term is a variable or a constant symbol.
A term is ground if it contains no variables.
An atom is a formula $p(t_1,\dots , t_n)$, where $p$ is a predicate symbol of arity $n$ and each $t_i$ is a term.
An atom is ground if all of its terms are ground.
A higher-order term is a higher-order variable or a predicate symbol. An atom is higher-order if it has at least one higher-order term.
A literal is an atom $A$ (a positive literal) or its negation $\neg A$ (a negative literal).
A clause is a disjunction of literals.
The variables in a clause are universally quantified.
A Horn clause is a clause with at most one positive literal.
A definite clause is a Horn clause with exactly one positive literal.
A clause is higher-order if it contains at least one higher-order atom.
A logic program is a set of Horn clauses.
A logic program is higher-order if it contains at least one higher-order Horn clause.

\subsection{Abstracted meta-interpretive learning}
\label{sec:amil}

We extend MIL to the higher-order setting. We first restate metarules \cite{crop:thesis}:

\begin{definition}[\textbf{Metarule}]
\label{def:metarule}
A metarule is a higher-order formula of the form:
$$\exists \pi \forall \mu \;\; l_0 \leftarrow l_1,\dots,l_m$$
\noindent
where each $l_i$ is a literal, $\pi \subseteq \mathcal{V}_1 \cup \mathcal{V}_2$, $\mu \subseteq \mathcal{V}_1 \cup \mathcal{V}_2$, and $\pi$ and $\mu$ are disjoint.
\end{definition}

\noindent
In contrast to a higher-order Horn clause, in which all the variables are all universally quantified, the variables in a metarule can be quantified universally or existentially\footnote{Existentially quantified first-order variables do not appear in this work, but do in existing work on MIL \cite{crop:datacurate}.}.
When describing metarules, we omit the quantifiers.
Instead, we denote existentially quantified higher-order variables as uppercase letters starting from $P$ and universally quantified first-order variables as uppercase letters starting from $A$.
Figure \ref{fig:metarules} shows example metarules.

\begin{figure}[ht]
\centering
\normalsize
\begin{tabular}{|l|l|}
\hline
\textbf{Name} & \textbf{Metarule}\\ \hline
ident & $P(A,B) \leftarrow Q(A,B)$\\
precon & $P(A,B) \leftarrow Q(A),R(A,B)$\\
curry & $P(A,B) \leftarrow Q(A,B,R)$\\
chain & $P(A,B) \leftarrow Q(A,C), R(C,B)$\\
\hline
\end{tabular}
\caption{Example metarules. The letters $P$, $Q$, and $R$ denote existentially quantified higher-order variables. The letters $A$, $B$, and $C$ denote universally quantified first-order variables.}
\label{fig:metarules}
\end{figure}

To extend MIL to support learning higher-order programs we introduce higher-order definitions:

\begin{definition}[\textbf{Higher-order definition}]
\label{def:hodef}
A higher-order definition is a set of higher-order Horn clauses where the head atoms have the same predicate symbol.
\end{definition}

\noindent
Three example higher-order definitions are:

\begin{example}[\textbf{Map definition}]
\label{example1}
\begin{center}
  \begin{tabular}{l}
    map([],[],F) $\leftarrow$\\
    map([A|As],[B|Bs],F) $\leftarrow$
      F(A,B),
      map(As,Bs)
  \end{tabular}
\end{center}
\end{example}

\noindent
In Example \ref{example1} the symbol $F$ is a universally quantified higher-order variable.
The other variables are universally quantified first-order variables.

\begin{example}[\textbf{Until definition}]
\label{example2}
\begin{center}
  \begin{tabular}{l}
    until(A,A,Cond,F) $\leftarrow$ Cond(A)\\
    until(A,B,Cond,F) $\leftarrow$ not(Cond(A)), F(A,C), until(C,B,Cond,F)
  \end{tabular}
\end{center}
\end{example}


\begin{example}[\textbf{Fold definition}]
\label{example3}
\begin{center}
  \begin{tabular}{l}
    fold([],Acc,Acc,F) $\leftarrow$\\
    fold([A|As],Acc1,B,F) $\leftarrow$ F(A,Acc1,Acc2), fold(As,Acc2,B,F)
  \end{tabular}
\end{center}
\end{example}

\noindent
We frequently refer to \emph{abstractions}.
In computer science code abstraction \cite{cardelli:abstraction} involves hiding complex code to provide a simpler interface.
In this work, we define an \emph{abstraction} as a higher-order Horn clause that contains at least one atom which takes a predicate symbol an argument.
In the following abstraction example, the final argument of $map/3$ is ground to the predicate symbol $succ/2$:

\begin{example}[\textbf{Abstraction}]
  \begin{tabular}{l}
    f(A,B) $\leftarrow$ map(A,B,succ)
  \end{tabular}
\end{example}

\noindent
Likewise, in the higher-order decryption program in the introduction (Figure \ref{fig:introho}), the final argument of \tw{map/3} is ground to the predicate symbol \tw{decrypt1/2}.


We define the abstracted MIL input, which extends a standard MIL input \cite{crop:thesis} (and problem) to support higher-order definitions:

\begin{definition}[\textbf{Abstracted MIL input}]
\label{def:milinput}
An abstracted MIL input is a tuple $(B,E^+,E^-,M)$ where:
\begin{itemize}
    \item $B=B_C \cup B_I$ where $B_C$ is a set of Horn clauses and $B_I$ is (the union of) a set of higher-order definitions
    \item $E^+$ and $E^-$ are disjoint sets of ground atoms representing positive and negative examples respectively
    \item $M$ is a set of metarules.
\end{itemize}
\end{definition}

\noindent
There is little declarative difference between $B_C$ and $B_I$.
There is, however, a procedural difference between the two.
We therefore call $B_C$ \emph{compiled} BK and $B_I$ \emph{interpreted} BK (IBK).
The procedural distinction between $B_C$ and $B_I$ is that whereas a clause from $B_C$ is proved deductively (by calling Prolog), a clause from $B_I$ is proved through meta-interpretation, which allows for predicate invention to be combined with abstractions to invent higher-order predicates.
The distinction between $B_I$ and $M$ is that the clauses in $B_I$ are all universally quantified, whereas the metarules in $M$ contain existentially quantified variables whose substitutions form the induced program.
We discuss these distinctions in more detail in Section \ref{sec:algos} when we describe the MIL implementations.


We define the abstracted MIL problem:

\begin{definition}[\textbf{Abstracted MIL problem}]
\label{def:milproblem}
Given an abstracted MIL input $(B,E^+,E^-,M)$, the abstracted MIL problem is to return a logic program hypothesis $H$ such that:
\begin{itemize}
  \item $\forall h \in H, \exists m \in M$ such that $h=m\theta$, where $\theta$ is a substitution that grounds all the existentially quantified variables in $m$
  \item $H \cup B \models E^{+}$
  \item $H \cup B \not\models E^{-}$
\end{itemize}
We call $H$ a solution to the MIL problem.
\end{definition}

\noindent
The first condition ensures that a logic program hypothesis is an instance of the given metarules.
It is this condition that enforces the strong inductive bias in MIL.

MIL supports inventions:

\begin{definition}[\textbf{Invention}]
\label{mil:def:inventions}
Let $(B,E^+,E^-,M)$ be a MIL input and $H$ be a solution to the MIL problem.
Then a predicate symbol $p/a$ is an \emph{invention} if and only if it is in the predicate signature (i.e. the set of all predicate symbols with their associated arities) of $H$ and not in the predicate signature of $B \cup E^+ \cup E^-$.
\end{definition}

\noindent
A MIL learner uses abstractions to generate inventions:

\begin{example}[\textbf{Invention}]
\begin{center}
\begin{tabular}{l}
  f(A,B) $\leftarrow$ map(A,B,f1)\\
  f1(A,B) $\leftarrow$ succ(A,C),succ(C,B)
\end{tabular}
\end{center}
\end{example}

\noindent
In this program, a MIL learner has invented the predicate \tw{f1/2} for use in a \tw{map/3} definition.
Likewise, in the higher-order decryption program in the introduction (Figure \ref{fig:introho}), the final argument of \tw{map/3} is ground to the invented predicate symbol \tw{decrypt1/2}.


\subsection{Language classes, expressivity, and complexity}
\label{sec:sc}

We claim that increasing the expressivity of MIL from learning first-order programs to learning higher-order programs can improve learning performance.
We support this claim by showing that learning higher-order programs can reduce the size of the hypothesis space which in turn reduces sample complexity and expected error.
In MIL the size of the hypothesis space is a function of the number of metarules $m$ and their form, the number of background predicate symbols $p$, and the maximum program size $n$ (the maximum number of clauses allowed in a program).
We restrict metarules by their body size and literal arity:

\begin{definition}[\textbf{Metarule fragment \M{i}{j}}]
\label{def:ham}
A metarule is in the fragment \M{i}{j} if it has at most $j$ literals in the body and each literal has arity at most $i$.
\end{definition}

\noindent
For instance, the \emph{chain} metarule in Figure \ref{fig:metarules} restricts clauses to be definite with two body literals of arity two, i.e. is in the fragment \M{2}{2}.
By restricting the form of metarules we can calculate the size of a MIL hypothesis space.
The following result is essentially the same as in \cite{crop:dreduce}.
The only difference is that we drop the redundant Big O notation:

\begin{proposition}[\textbf{MIL hypothesis space}]
\label{prop:hs1}
Given $p$ predicate symbols and $m$ metarules in \M{i}{j}, the number of programs expressible with $n$ clauses is at most $(mp^{j+1})^n$.
\end{proposition}


\begin{proof}
The number of clauses which can be constructed from a \M{i}{j} metarule given $p$ predicate symbols is at most $p^{j+1}$ because for a given metarule there are at most $j+1$ predicate variables with at most $p^{j+1}$ possible substitutions.
Therefore the number of clauses that can be formed from $m$ distinct metarules in \M{i}{j} using $p$ predicate symbols is at most $mp^{j+1}$.
It follows that the number of programs which can be formed from a selection of $n$ such clauses is at most $(mp^{j+1})^n$.
\end{proof}

\noindent
Proposition \ref{prop:hs1} shows that the MIL hypothesis space grows exponentially both in the size of the target program and the number of body literals in a clause.
For instance, for the \M{2}{2} fragment, the MIL hypothesis space contains at most $(mp^3)^n$ programs, where $m$ is the number of metarules and $n$ is the number of clauses in the target program.




We update this bound for the abstracted MIL framework:


\begin{proposition}[\textbf{Number of abstracted \M{i}{j} programs}]
\label{prop:hs2}
Given $p$ predicate symbols and $m$ metarules in \M{i}{j} with at most $k$ additional existentially quantified higher-order variables,
the number of abstracted \M{i}{j} programs expressible with $n$ clauses is at most $(mp^{j+1+k})^n$.
\end{proposition}
\begin{proof}
As with Proposition \ref{prop:hs1}, the number of clauses which can be constructed from a \M{i}{j} metarule given $p$ predicate symbols is at most $p^{j+1}$ because for a given metarule there are at most $j+1$ predicate variables with at most $p^{j+1}$ possible substitutions.
Given a metarule in \M{i}{j} with at most $k$ additional existentially quantified higher-order variables there are therefore potentially $j+1+k$ predicate variables with $p^{j+1+k}$ possible substitutions.
Therefore the number of clauses expressible with $m$ such metarules is at most $mp^{j+1+k}$.
By the same reasoning as for Proposition \ref{prop:hs1}, it follows that the number of programs which can be formed from a selection of $n$ such clauses is at most $(mp^{j+1+k})^n$.
\end{proof}

\noindent
We use this result to develop sample complexity \cite{mitchell:ml} results for unabstracted MIL:

\begin{proposition}[\textbf{Sample complexity of unabstracted MIL}]
\label{prop:sc1}
Given $p$ predicate symbols, $m$ metarules in \M{i}{j}, and a maximum program size $n_u$, unabstracted MIL has sample complexity:
$$s_u \geq  \frac{1}{\epsilon} (n_u\; \ln(m) + (j+1)n_u \;  \ln(p) + \ln(\frac{1}{\delta}))$$
\end{proposition}
\begin{proof}
According to the Blumer bound, which appears as a reformulation of Lemma 2.1 in \cite{blumer:bound}, the error of consistent hypotheses is bounded by $\epsilon$ with probability at least $(1-\delta)$ once $s_u \geq \frac{1}{\epsilon} (\ln(|H|) + \ln(\frac{1}\delta))$, where $|H|$ is the size of the hypothesis space.
From Proposition \ref{prop:hs1}, $|H| = (mp^{j+1})^{n_u}$ for unabstracted MIL.
Substituting and applying logs gives:
\[s_u \geq  \frac{1}{\epsilon} (n_u\; \ln(m) + (j+1)n_u \;  \ln(p) + \ln(\frac{1}{\delta}))\]
\end{proof}

\noindent
We likewise develop sample complexity results for abstracted MIL:

\begin{proposition}[\textbf{Sample complexity of abstracted MIL}]
\label{prop:sc2}
Given $p$ predicate symbols, $m$ metarules in \M{i}{j} augmented with at most $k$ higher-order variables, and a maximum program size $n_a$, abstracted MIL has sample complexity:
$$s_a \geq \frac{1}{\epsilon}(n_a \; \ln(m) + (j+1+k)n_a \; \ln(p) + \ln(\frac{1}{\delta}))$$
\end{proposition}
\begin{proof}
Analogous to Proposition \ref{prop:sc1} using Proposition \ref{prop:hs2}.
\end{proof}

\noindent
We compare these bounds:

\begin{theorem}[\textbf{Unabstracted and abstracted bounds}]
\label{thm:ratio}
Let $m$ be the number of \M{i}{j} metarules, $n_u$ and $n_a$ be the minimum numbers of clauses necessary to express a target theory with unabstracted and abstracted MIL respectively, $s_u$ and $s_a$ be the bounds on the number of training examples required to achieve error less than $\epsilon$ with probability at least $1-\delta$ with unabstracted and abstracted MIL respectively, and $k\geq1$ be number of additional higher-order variables used by abstracted MIL.
Then $s_u > s_a$ when:
\[n_u - n_a > \dfrac{k}{j+1}n_a\]

\end{theorem}
\begin{proof}
From Proposition \ref{prop:sc1} it holds that:
\[s_u \geq  \frac{1}{\epsilon} (n_u\; \ln(m) + (j+1)n_u \;  \ln(p) + \ln(\frac{1}{\delta}))\]
From Proposition \ref{prop:sc2} it holds that:
\[s_a \geq \frac{1}{\epsilon}(n_a \; \ln(m) + (j+1+k)n_a \; \ln(p) + \ln\frac{1}{\delta})\]
If we cancel $\frac{1}{\epsilon}$ then $s_u > s_a$ follows from:
\[n_u \ln(m) + (j+1)n_u \ln(p) > n_a \ln(m) + (j+1+k)n_a \ln(p)\]
Because $k\geq1$, the inequality $s_u > s_a$ holds when:
\begin{equation} \tag{1}
n_u \ln(m) > n_a \ln(m)
\end{equation}
and:
\begin{equation} \tag{2}
(j+1)n_u \ln(p) > (j+1+k)n_a \ln(p)
\end{equation}
Because $k \geq 1$ the inequality (2) implies the inequality (1).
The inequality (2) holds when $(j+1)n_u > (j+1+k)n_a$.
Therefore $s_u > s_a$ follows from $(j+1)n_u > (j+1+k)n_a$.
Rearranging terms leads to $s_u > s_a$ when $n_u - n_a > \frac{k}{j+1}n_a$.
\end{proof}

\noindent
The results from this section motivate the use of abstracted MIL, and help explain the experimental results (Section \ref{sec:experiments}).
To illustrate these theoretical results, reconsider the decryption programs shown in Figure \ref{fig:intro}.
Consider representing these programs in \M{2}{2}.
Figure \ref{fig:m22fo} shows that the first-order program would require seven clauses.
By contrast, Figure \ref{fig:m22ho} shows that the higher-order program requires only three clauses and one extra higher-order variable.
Let $m_u = 4$, $p_u=6$, and $n_u=7$ be the number of metarules, background relations, and clauses needed to express the first-order program shown in Figure \ref{fig:m22fo}.
Plugging these values into the formula in Proposition \ref{prop:hs1} shows that the hypothesis space of unabstracted MIL contains approximately $10^{21}$ programs.
By contrast, suppose we allow an abstracted MIL learner to additionally use the higher-order definition \tw{map/3} and the corresponding \emph{curry} metarule $P(A,B) \leftarrow Q(A,B,R)$.
Therefore $m_a = m_u+1$, $p_a=p_u+1$, $n_a=3$, and $k=1$, where $k$ is the number of additional higher-order variables used in the curry metarule.
Then plugging these values into the formula from Proposition \ref{prop:hs2} shows that the hypothesis space of abstracted MIL contains approximately $10^{13}$ programs, which is substantially smaller than the first-order hypothesis space, despite using more metarules and more background relations.
The Blumer bound \cite{blumer:bound} says that given two hypothesis spaces of different sizes, then searching the smaller space will result in less error compared to the larger space, assuming that the target hypothesis is in both spaces.
In this example, the target hypothesis, or a hypothesis that is equivalent\footnote{Success set equivalent when restricted to the target predicate \tw{decrypt/2}. The success set of a logic program $P$ is the set of ground atoms $\{A \in hb(P)|P\cup\{ \neg A \}\;\text{has a SLD-refutation}\}$, where $hb(P)$ represents the Herband base of the logic program $P$. The success set restricted to a specific predicate symbol $p$ is the subset of the success set restricted to atoms of the predicate symbol $p$.} to the target hypothesis, is in both hypothesis spaces but the abstracted MIL space is smaller.
Therefore, our results imply that in this scenario, given a fixed number of examples, abstracted MIL should improve predictive accuracies compared to unabstracted MIL.
In Section \ref{sec:encryption} we experimentally explore whether this result holds.

\begin{figure}[ht]
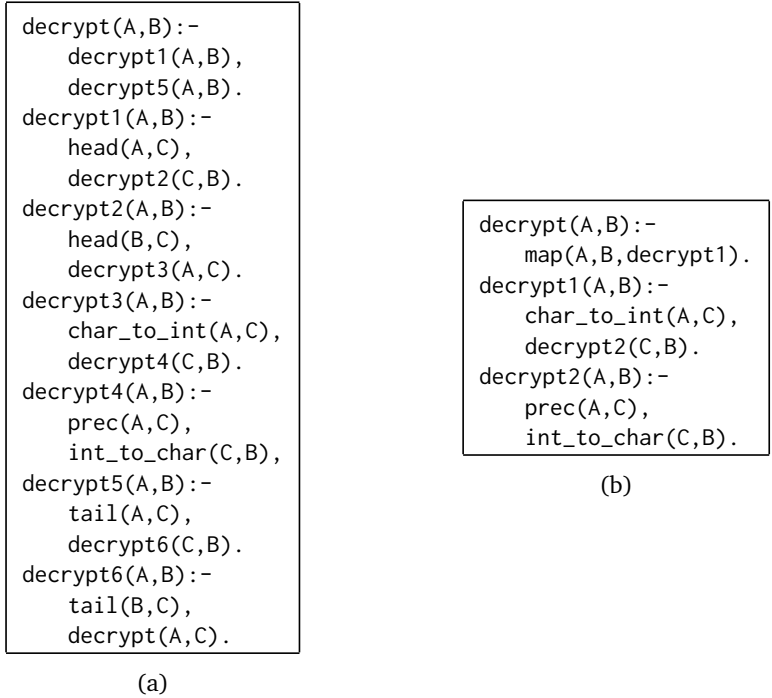

\normalsize
\begin{subfigure}[t]{.5\textwidth}
\centering
\normalsize
\begin{tabular}{|c|}
\hline
\begin{lstlisting}
decrypt(A,B):-
    decrypt1(A,B),
    decrypt5(A,B).
decrypt1(A,B):-
    head(A,C),
    decrypt2(C,B).
decrypt2(A,B):-
    head(B,C),
    decrypt3(A,C).
decrypt3(A,B):-
    char_to_int(A,C),
    decrypt4(C,B).
decrypt4(A,B):-
    prec(A,C),
    int_to_char(C,B),
decrypt5(A,B):-
    tail(A,C),
    decrypt6(C,B).
decrypt6(A,B):-
    tail(B,C),
    decrypt(A,C).
\end{lstlisting} \\
\hline
\end{tabular}
\normalsize
\caption{}
\label{fig:m22fo}
\end{subfigure}%
\begin{subfigure}[t]{.5\textwidth}
\centering
\begin{tabular}{|c|}
\hline
\begin{lstlisting}
decrypt(A,B):-
    map(A,B,decrypt1).
decrypt1(A,B):-
    char_to_int(A,C),
    decrypt2(C,B).
decrypt2(A,B):-
    prec(A,C),
    int_to_char(C,B).
\end{lstlisting} \\
\hline
\end{tabular}
\caption{}
\label{fig:m22ho}
\end{subfigure}
\caption{Decryption programs.
Figure (a) shows a first-order program represented in \M{2}{2}.
Figure (b) shows a higher-order program represented in \M{2}{2} with one extra higher-order variable (the third argument of \tw{map/3}).
}
\end{figure}
\section{Algorithms}
\label{sec:algos}

We now introduce \metaho{} and \hexho{}, both of which implement abstracted MIL and which extend Metagol and \hex{} respectively.
For self-containment, we also describe Metagol and \hex{}.

\subsection{Metagol}
\label{sec:metagol}
Metagol \cite{metagol} is a MIL learner based on a Prolog meta-interpreter.
Figure \ref{fig:metagol} shows Metagol's learning procedure described using Prolog.
Metagol works as follows.
Given a set of atoms representing positive examples, Metagol tries to prove each atom in turn.
Metagol first tries to deductively prove an atom using compiled BK by delegating the proof to Prolog (\tw{call(Atom)}), where the compiled BK contains standard Prolog definitions.
Metagol uses \emph{prim} statements to allow a user to specify what predicates are part of the compiled BK.
Prim statements of the form \tw{prim(P/A)}, where \tw{P} is a predicate symbol and \tw{A} is the associated arity, and are similar to determinations used by Aleph \cite{aleph}, except that Metagol only requires prim statements for predicates that may appear in the body.
If this deductive step fails, Metagol tries to unify the atom with the head of a metarule (\tw{metarule(Name,Subs,(Atom:-Body))}) and tries to bind the existentially quantified higher-order variables in a metarule to symbols in the predicate signature, where \tw{Subs} contains the substitutions.
Metagol saves the resulting substitutions and tries to prove the body of the metarule.
After proving all atoms, a Prolog program is formed by projecting the substitutions onto their corresponding metarules.
Metagol checks the consistency of the learned program with the negative examples.
If the program is inconsistent, then Metagol backtracks to explore different branches of the SLD-tree.

Metagol uses iterative deepening to ensure that the first consistent hypothesis returned has the minimal number of clauses.
The search starts at depth 1.
At depth $d$ the search returns a consistent hypothesis with at most $d$ clauses if one exists; otherwise it continues to depth $d+1$.
At each depth $d$, Metagol introduces $d-1$ new predicate symbols\footnote{
    Metagol forms new predicate symbols by taking the name of the task and adding underscores and numbers.
For example, if the task is \tw{f} and the depth is $4$ then Metagol will add the predicate symbols \tw{f\_3}, \tw{f\_2}, and \tw{f\_1} to the predicate signature.
Note that in this paper we remove the underscore symbols from any learned programs to save space, but the experimental code contains the original underscore symbols.
}.



\begin{figure}[ht]
\centering
\begin{tabular}{|c|}
\hline
\begin{lstlisting}
learn(Pos,Neg,Prog):-
    prove(Pos,[],Prog),
    \+ prove(Neg,Prog,Prog).
prove([],Prog,Prog).
prove([Atom|Atoms],Prog1,Prog2):-
    prove_aux(Atom,Prog1,Prog3),
    prove(Atoms,Prog3,Prog2).
prove_aux(Atom,Prog,Prog):-
    prim(Atom),!,
    call(Atom).
prove_aux(Atom,Prog1,Prog2):-
    member(sub(Name,Subs),Prog1),
    metarule(Name,Subs,(Atom:-Body)),
    prove(Body,Prog1,Prog2).
prove_aux(Atom,Prog1,Prog2):-
    metarule(Name,Subs,(Atom:-Body)),
    prove(Body,[sub(Name,Subs)|Prog1],Prog2).
\end{lstlisting} \\
\hline
\end{tabular}
\caption{
Metagol's learning procedure described using Prolog.
Note that this code is the barebones code for Metagol and the actual code differs.
The actual code has slightly different syntax and includes more code, such as code to perform the iterative deepening and code to invent new predicate symbols.
For instance, in this Figure \tw{prim(Atom)} is not of the form \tw{prim(P/A)}, as described in the text.
}
\label{fig:metagol}
\end{figure}

\subsection{\metaho{}}

Figure \ref{fig:metagolho} shows the Prolog code for \metaho{}.
The key difference between \metaho{} and Metagol is the introduction of the second \tw{prove\_aux/3} clause in the meta-interpreter, denoted in boldface.
This clause allows \metaho{} to prove an atom by fetching a clause from the IBK (such as \tw{map/3}) whose head unifies with a given atom.
The distinction between compiled and interpreted BK is that whereas a clause from the compiled BK is proved deductively by calling Prolog, a clause from the IBK is proved through meta-interpretation.
Meta-interpretation allows for predicate invention to be driven by the proof of conditions (as in \tw{filter/3}) and functions (as in \tw{map/3}).
IBK is different to metarules because the clauses are all universally quantified and, importantly, does not require any substitutions.
By contrast, metarules contain existentially quantified variables whose substitutions form the hypothesised program.
Figure \ref{fig:bk} shows examples of the three forms of BK used by \metaho{}.

\begin{figure}[ht]
\centering
\begin{tabular}{|c|}
\hline
\begin{lstlisting}
learn(Pos,Neg,Prog):-
    prove(Pos,[],Prog),
    \+ prove(Neg,Prog,Prog).
prove([],Prog,Prog).
prove([Atom|Atoms],Prog1,Prog2):-
    prove_aux(Atom,Prog1,Prog3),
    prove(Atoms,Prog3,Prog2).
prove_aux(Atom,Prog,Prog):-
    prim(Atom),!,
    call(Atom).
(*\bfseries prove\_aux(Atom,Prog1,Prog2):-*)
    (*\bfseries ibk((Atom:-Body)),*)
    (*\bfseries prove(Body,Prog1,Prog2).*)
prove_aux(Atom,Prog1,Prog2):-
    member(sub(Name,Subs),Prog1),
    metarule(Name,Subs,(Atom:-Body)),
    prove(Body,Prog1,Prog2).
prove_aux(Atom,Prog1,Prog2):-
    metarule(Name,Subs,(Atom:-Body)),
    prove(Body,[sub(Name,Subs)|Prog1],Prog2).
\end{lstlisting} \\
\hline
\end{tabular}
\caption{Prolog code for \metaho{}.
}
\label{fig:metagolho}
\end{figure}

\begin{figure}[ht]
\centering

\begin{minipage}{.35\textwidth}
\centering
\begin{lstlisting}[frame=single,title={Compiled BK}]
empty([]).
head([H|_],H).
tail([_|T],T).
last([A],A):-!.
last([_|T],A):-last(T,A).
\end{lstlisting}
\end{minipage}

\begin{minipage}{.95\textwidth}
\centering
\begin{lstlisting}[frame=single,title={Interpreted BK}]
ibk(([map,[],[],F]:-[])).
ibk(([map,[A|As],[B|Bs],F]:-[[F,A,B],[map,As,Bs,F]])).
ibk(([fold,[],Acc,Acc,F]:-[])).
ibk(([fold,[A|As],Acc1,B,F]:-[[F,A,Acc1,Acc2],[fold,As,Acc2,B,F]])).
ibk([ifthenelse,A,B,Cond,Then,_],[[Cond,A],[Then,A,B]]).
ibk([ifthenelse,A,B,Cond,_,Else],[[not,Cond,A],[Else,A,B]]).
\end{lstlisting}
\end{minipage}

\begin{minipage}{.76\textwidth}
\centering
\begin{lstlisting}[frame=single,title={Metarules},label=lst:metarules]
metarule(monadic,[P,Q],([P,A,A]:-[[Q,A]])).
metarule(identity,[P,Q],([P,A,B]:-[[Q,A,B]])).
metarule(inverse,[P,Q],([P,A,B]:-[[Q,B,A]])).
metarule(didentity,[P,Q],([P,A,B]:-[[Q,A,B],[R,A,B]])).
metarule(precon,[P,Q,R],([P,A,B]:-[[Q,A],[R,A,B]])).
metarule(postcon,[P,Q,R],([P,A,B]:-[[Q,A,B],[R,B]])).
metarule(curry1,[P,Q,R],([P,A,B]:-[[Q,A,B,R]])).
metarule(curry2,[P,Q,R,S],([P,A,B]:-[[Q,A,B,R,S]])).
metarule(curry3,[P,Q,R,S,T],([P,A,B]:-[[Q,A,B,R,S,T]])).
metarule(chain,[P,Q,R],([P,A,B]:-[[Q,A,C],[R,C,B]])).
metarule(tailrec,[P,Q],([P,A,B]:-[[Q,A,C],[P,C,B]])).
\end{lstlisting}
\end{minipage}

\caption{
Three forms of BK used by \metaho{} described in Prolog syntax.
The curry rules are slightly unusual but are necessary to use the interpreted BK (e.g. \emph{curry1} allows us to use the \tw{map/3} definition).}
\label{fig:bk}
\end{figure}

\metaho{} works in the same way as Metagol except for the use of IBK.
\metaho{} first tries to prove an atom deductively using compiled BK by delegating the proof to Prolog (\tw{call(Atom)}), exactly how Metagol works.
If this step fails, \metaho{} tries to unify the atom with the head of a clause in the IBK (\tw{ibk((Atom:-Body))}) and tries to prove the body of the matched definition.
Metagol does not perform this additional step.
Failing this, \metaho{} continues to work in the same way as Metagol.
\metaho{} uses negation as failure \cite{Clark87} to negate predicates in the compiled BK.
Negation of invented predicates is unsupported and is left for future work\footnote{Metagol could support the negation of invented predicates but it is non-trivial to efficiently negate an invented predicate that itself contains an invented predicate. This limitation could be addressed by allowing Metagol to alternatively perform a generate and test approach. However, a generate-and-test approach is impractical for non-trivial situations.}.

To illustrate the difference between Metagol and \metaho{}, suppose you have compiled BK containing the \tw{succ/2}, \tw{int\_to\_char/2}, and \tw{map/3} predicates and the \emph{curry1} ($P(A,B) \leftarrow Q(A,B,R)$) and \emph{chain} ($P(A,B) \leftarrow Q(A,C), R(C,B)$) metarules.
Suppose you are given the examples \tw{f([1,2,3],[c,d,e])} and \tw{f([1,2,1],[c,d,c])} where the underlying target hypothesis is to add two to each element of the list and find the corresponding letter in an a-z index.
Given these examples Metagol would try to prove each atom in turn.
Metagol cannot prove any example using only the compiled BK so it would need to use a metarule.
Suppose it unifies the atom \tw{f([1,2,3],[c,d,e])} with the \emph{curry} metarule.
Then the new atom to prove would be \tw{Q([1,2,3],[c,d,e],R)}.
To prove this atom Metagol could unify \tw{map/3} with \tw{Q} and then try to prove the atom \tw{map([1,2,3],[c,d,e],R)}.
However, the proof of \tw{map([1,2,3],[c,d,e],R)} would fail because there is no suitable substitution for \tw{R}.
The only possible substitution for $R$ is \tw{succ/2}, which will clearly not allow the proof to succeed.
The only way Metagol can learn a consistent hypothesis is by successively chaining calls to \tw{map(A,B,succ)} and \tw{map(A,B,int\_to\_char)} using the \emph{chain} metarule to learn:

\begin{center}
\begin{minipage}{.40\textwidth}
\centering
\begin{lstlisting}[frame=single]
f(A,B):-f1(A,C),f3(C,B)
f1(A,B):-f2(A,C),f2(C,B).
f2(A,B):-map(A,B,succ).
f3(A,B):-map(A,B,int_to_char).
\end{lstlisting}
\end{minipage}
\end{center}

\noindent
By contrast, suppose we had the same setup for \metaho{} but we allowed \tw{map/3} to be defined as IBK.
In this case, \metaho{} would unify the atom \tw{f([1,2,3],[c,d,e])} with the \emph{curry1} metarule.
The new atom to prove would therefore be \tw{Q([1,2,3],[c,d,e],R)}.
In contrast to Metagol, \metaho{} can unify this atom with \tw{map/3} defined as IBK.
\metaho{} will then try to prove \tw{map([1,2,3],[c,d,e],R)} through meta-interpretation.
This step would result in a sequence of new atoms to prove \tw{R(1,c), R(2,d), R(3,e)}.
These new atoms can also be proven though meta-interpretation which allows for \metaho{} to invent and define the suitable symbol for \tw{R}.
Therefore, in this scenario, Metagol would learn:

\begin{center}
\begin{minipage}{.5\textwidth}
\centering
\begin{lstlisting}[frame=single]
f(A,B):-map(A,B,f1).
f1(A,B):-succ(A,C),f2(C,B).
f2(A,B):-succ(A,C),int_to_char(C,B).
\end{lstlisting}
\end{minipage}
\end{center}

\noindent
As this scenario illustrates, the real power and novelty of \metaho{} is the combination of abstraction (learning higher-order programs) and invention (predicate invention).
In this scenario, abstraction has allowed the atom \tw{Q([1,2,3],[c,d,e],R)} to be decomposed into the sub-problems \tw{R(1,c), R(2,d), R(3,e)}.
Further abstraction and invention allows for \metaho{} to solve these sub-problems by inventing and defining the necessary predicate for \tw{R}.
By successively interleaving these two steps, \metaho{} supports the invention of conditions and functions to an arbitrary depth, which goes beyond anything in the literature.

\subsection{\hex{}}

Before describing \hexho{}, which supports learning higher-order logic programs, first we discuss \hex{}, on which \hexho is based.

\hex{} is an answer set programming (ASP) encoding of MIL introduced by Kaminski et al. \cite{hexmil}.
Whereas Metagol searches for a proof (and thus a program) using a meta-interpreter and SLD-resolution, HEXMIL searches for a proof by encoding the MIL problem as an ASP problem.
As argued by Kaminski et al., an ASP implementation can be more efficient than a Prolog implementation because ASP solvers employ efficient conflict propagation, which is important for detecting the derivability of negative examples early during ASP search.

The \hex{} encoding specifies constraints on possible hypotheses derived from the examples, in addition to rules specifying the available BK.
An ASP solver performs a combinatorial search for solutions satisfying these constraints.
ASP solvers typically work in two phases: (1) a grounding phase, where rules are grounded, and (2) a solving phase, where reasoning on (propositional) rules leads to answer sets \cite{GelfondL91}.
A straightforward ASP encoding of the MIL problem is infeasible in many cases, for reasons such as the grounding bottleneck of ASP and the difficulty in manipulating complex structures such as lists \cite{hexmil}.
To mitigate these difficulties HEXMIL uses the HEX formalism \cite{eiter:hex} which allows ASP programs to interface with \emph{external sources}.
External sources are predicate definitions given by programs outside of the ASP language.
For instance, HEXMIL interfaces with external sources described as a Python program.
HEX programs can access these definitions via \emph{external atoms}.
HEXMIL benefits from external atoms by allowing for arbitrary encodings of complex structures (e.g.~we encode lists as strings, thereby reducing the number of variables needed in the encoding).
Another benefit is that external atoms allow for the incremental introduction of new constants (i.e. symbols not in the initial ASP program).

To improve efficiency, Kaminski et al.~introduced a \emph{forward-chained} HEXMIL-encoding which requires forward-chained metarules:

\begin{definition}[\textbf{Forward-chained metarule}]
\label{def:fwdchain}
A metarule is \emph{forward-chained} when it can be written in the form:
\[
P(A,B) \leftarrow Q_1(A,C_1),Q_2(C_1,C_2),\ldots,Q_i(C_{i-1},B),R_1(D_1),\ldots,R_j(D_j)
\]
where $D_1,\ldots,D_j$ are all contained in $\{A,C_1,\ldots,C_{i-1},B\}$.
\end{definition}

\noindent
In the forward-chained HEXMIL encoding, compiled (first-order) BK is encoded using the external atoms \tw{\&bkUnary[P,A]()} and \tw{\&bkBinary[P,A](B)}.
These two atoms represent all BK predicates of the form \tw{P(A)} and \tw{P(A,B)}, where \tw{P} and \tw{A} are input arguments to the external source and \tw{B} is an output argument.
Using the input/output ordering of the external binary atoms, grounding of variables in forward-chained metarules occurs from left to right.
HEXMIL uses the forward-chained encoding:

\begin{center}
\begin{tabular}{l}
\emph{deduced(P,A) $\leftarrow$ \&bkUnary[P,A](), state(A)}\\
\emph{deduced(P,A,B) $\leftarrow$ \&bkBinary[P,A](B), state(A)}\\
\emph{state(A) $\leftarrow$ for each P(A,B) $\in E^+ \cup E^-$}\\
\emph{state(B) $\leftarrow$ deduced(P,A,B)}\\
\end{tabular}
\end{center}

\noindent
HEXMIL uses the \emph{deduced} predicate to represent facts that hypotheses could entail.
In this encoding, the import of BK is guarded by the predicate \emph{state/1}.
A solution for MIL problem (Definition \ref{def:milproblem}) must entail all positive examples (i.e ground atoms).
Therefore, in HEXMIL, every positive examples must appear in the head of a grounded metarule.
It follows that ground terms in atoms can be seen as the states that can be reached from the examples.
Therefore, HEXMIL initially marks the ground terms that appear in the examples as \emph{state}.
As new ground terms are introduced by the external atoms, HEXMIL marks these values as \emph{state} as well.

To support metarules HEXMIL employs two encoding rules.
The first rule encodes the possible instantiations of a metarule.
Let $mr$ be the name of an arbitrary forward-chained metarule (Def.~\ref{def:fwdchain}), then for each such metarule, the first encoding rule is:

\begin{center}
\begin{tabular}{l}
$meta(mr,P,Q_1,\ldots,Q_i,R_1,\ldots,R_j)~\vee~neg\_meta(mr,P,Q_1,\ldots,Q_i,R_1,\ldots,R_j) \leftarrow $\\
$\hspace{2.3em} sig(P),sig(Q_1),\ldots,sig(Q_i),sig(R_1),\ldots,sig(R_j),$\\
$\hspace{2.3em} ord(P,Q_1),\ldots,ord(P,Q_i),ord(P,R_1),\ldots,ord(P,R_j),$\\
$\hspace{2.3em} deduced(Q_1,A,C_1),\ldots,deduced(Q_i,C_{i-1},B),$\\
$\hspace{2.3em} deduced(R_1,D_1),\ldots,deduced(R_j,D_j)$
\end{tabular}
\end{center}

\noindent
Note that the head in this rule allows for choosing whether to deduce the metarule instantiation.
Also note that the disjunction in the head means that this is not a Horn clause, yet it encodes a Horn clause metarule.
This encoding rule relies on two other rules:
\begin{center}
\begin{tabular}{l}
\emph{sig(p) $\leftarrow$ for each p $\in \mathcal{P}$} \\
\emph{ord(p,q) $\leftarrow$ for all p,q $\in \mathcal{P}$ s.t. p $\preceq$ q}\\
\end{tabular}
\end{center}
\noindent
The $sig$ relation denotes predicate symbols available, both invented and given as part of the BK.
The $ord$ relation denotes an ordering $\preceq$ over the predicate symbols.
This ordering disallows certain instantiations\footnote{Details on the $\preceq$-relation can be found in the paper on HEXMIL \cite{hexmil} as well as in the files on our experimental work.}, e.g.~recursive instantiations.


The second metarule encoding allows for metarule instantiations to be generated in order to derive facts:
\begin{center}
\begin{tabular}{l}
$deduced(P,A,B) \leftarrow$\\
$\hspace{2.1em} meta(mr,P,Q_1,\ldots,Q_i,R_1,\ldots,R_j),$\\
$\hspace{2.1em} deduced(Q_1,A,C_1),\ldots,deduced(Q_i,C_{i-1},B),$\\
$\hspace{2.1em} deduced(R_1,D_1),\ldots,deduced(R_j,D_j)$
\end{tabular}
\end{center}

\noindent
The generation of metarule instantiations are then checked by the solver for consistency with the examples.
This checking step relies on constraints derived from positive and negative examples:

\begin{center}
\begin{tabular}{l}
$\leftarrow not \; deduce(P,A,B) \hspace{2mm} \textnormal{for each~} P(A,B) \in E^+$\\
$\leftarrow deduce(P,A,B) \hspace{2mm} \textnormal{for each~} P(A,B) \in E^-$
\end{tabular}
\end{center}

\noindent
Similar to Metagol, HEXMIL searches for solutions using iterative deepening on the number of allowed metarule instantiations and the number of predicate symbols.
We omit the details of the ASP constraints that restrict the number of metarule instantiations.

\subsection{\hexho{}}

We now describe the extension of \hex{} to \hexho{}, which adds support for higher-order definitions, i.e.~interpreted background knowledge (IBK).
This extension allows HEXMIL to search for programs in \emph{abstracted} forward-chained hypothesis spaces.
To extend HEXMIL, we introduce a new predicate $ibk$ to encode the higher-order atoms that occur in IBK.
Note that $ibk$ is a normal ASP predicate and not an external atom.
This predicate allows us to encode higher-order clauses as a mix of $deduced$ atoms for first-order predicates and $ibk$ atoms for those that involve predicates as arguments.

Let the following be a clause of an arbitrary (forward-chained) higher-order definition:
\begin{center}
\begin{tabular}{l}
$h(A,B,P_{0,1},\ldots,P_{0,k_0}) \leftarrow h_1(A,C_1,P_{1,1},\ldots,P_{1,k_1}),\ldots,h_j(C_{j-1},B,P_{j,1},\ldots,P_{j,k_j})$\\
\end{tabular}
\end{center}
Every atom in this clause can have $0 \leq k_i$ higher-order terms.
The higher-order clauses of the definition will have at least one atom with $k_i \neq 0$.
For each clause in a higher-order definition we give a rule encoding the clause,
where $C_0 = A$ and $C_j = B$:

\begin{center}
\begin{tabular}{ll}
$ibk(h,A,B,P_{0,1},\ldots,P_{0,k_0}) \leftarrow$\\
$\hspace{2.35em} state(A),$\\
$\hspace{2.35em} sig(P_{0,1}),\ldots,sig(P_{0,k_0}),$\\
$\hspace{2.35em} ibk(h_i,C_{i-1},C_i,P_{i,1},\ldots,P_{i,k_i}),sig(P_{i,1}),\ldots,sig(P_{i,k_i})$&if $k_i > 0$\\
$\hspace{2.35em} deduced(h_i,C_{i-1},C_i)$&if $k_i = 0$\\
\end{tabular}
\end{center}
\noindent
Figure \ref{fig:hexmil-ex} shows an example of this encoding for the \tw{until/4} predicate.
Figure \ref{fig:hexmil-ex} also contains a definition for \tw{map/3} (which is slightly more involved).
This approach to higher-order definitions also applies to metarules involving higher-order atoms.
For instance, Figure \ref{fig:hexmil-ex} also shows the encoding of the \emph{curry2} metarule.

Our extension is sufficient\footnote{Kaminski, et al.~\cite{hexmil} proposed an additional first-order \emph{state-abstraction} encoding that improved the efficiency of the learning. It is currently unclear as how to integrate IBK into this encoding.}
to learn higher-order programs.
Note that in this setting higher-order definitions are required to be forward-chained in their first-order arguments, meaning that left-to-right grounding of these arguments is still valid.
The remaining (higher-order) arguments can be ground by the $sig$ predicate, which contains all the predicate names.
As predicate symbols were already arguments in the \hex{} encoding, we can easily make a predicate argument occur as an atom's predicate symbol, e.g.~see the variable \tw{F} in \tw{until/4} and \tw{map/3} in Figure \ref{fig:hexmil-ex}.


\begin{figure}[ht]
\centering
\begin{minipage}{.70\textwidth}
\centering
\begin{lstlisting}[mathescape=true,frame=single,title={\emph{curry2} metarule},label=lst:aspcurry2]
meta(curry2,P,Q,R,S)$~\vee~$neg_meta(curry2,P,Q,R,S):-
    sig(P),sig(Q),sig(R),sig(S),
    ord(P,Q),ord(P,R),ord(P,S),
    ibk(Q,A,B,R,S).
deduced(P,A,B):-
    meta(curry2,P,Q,R,S),
    ibk(Q,A,B,R,S).
\end{lstlisting}
\end{minipage}


\begin{minipage}{.375\textwidth}
\centering
\begin{lstlisting}[mathescape=true,frame=single,title={\tw{until/4}},label=lst:aspuntil]
ibk(until,A,A,Cond,F):-
    state(A),
    sig(Cond),
    sig(F),
    deduced(Cond,A).
ibk(until,A,B,Cond,F):-
    state(A),
    sig(Cond),
    sig(F),
    not deduced(Cond,A),
    deduced(F,A,C),
    ibk(until,C,B,Cond,F).
\end{lstlisting}
\end{minipage}
\hspace{4em}
\begin{minipage}{.35\textwidth}
\centering
\begin{lstlisting}[mathescape=true,frame=single,title={\tw{map/3}},label=lst:aspmap]
ibk(map,"[]","[]",F):-
    sig(F).
ibk(map,L1,L2,F):-
    state(L1),
    sig(F),
    deduced(head,L1,H1),
    deduced(tail,L1,T1),
    deduced(F,H1,H2),
    ibk(map,T1,T2,F),
    &prepend[H2,T2](L2).
\end{lstlisting}
\end{minipage}
\caption{\hexho{} code examples. The "[]" symbol in the \tw{map/3} definition is special syntax we use to represent lists. Note that due to lists being encoded as strings, the \tw{prepend} external atom is required to manipulate the lists in the \tw{map/3} definition.}
\label{fig:hexmil-ex}
\end{figure}


\subsection{Complexity of the search}
\label{sec:complex}


The experiments in the next section use both \metagol{} and \hex{}, and their higher-order extensions.
The purpose of the experiments is to test our claim that learning higher-order programs, rather than first-order programs, can improve learning performance.
Although we do not directly compare them, the experimental results show a significant difference in the learning performances of \metagol{} and \hex{}, and their higher-order variants.
The experimental results also show that \hex{} and \hexho{} do not scale well, both in terms of the amount of BK and the number of training examples.
To help explain these results, we now contrast the theoretical complexity of \metagol{} and \hex{}.
For simplicity we focus on the $\mathcal{M}_2^2$ hypothesis space, although our results can easily be generalised.
Our main observation is that the performance of \hex{} is a function of the number of constant symbols, which is not the case for \metagol{}.

From Proposition \ref{prop:hs1} it follows that the $\mathcal{M}_2^2$ MIL hypothesis space contains at most $(mp^3)^n$ programs.
For \metagol{}, this bound is an over-approximation on the number of programs that will be considered during the search.
Given a training example, \metagol{} learns a program by trying different substitutions for the existentially quantified predicate symbols in metarules, where the search is driven by the example.
Metagol only considers constants that it encounters when it evaluates whether a hypothesis covers an example, in which case it only considers the constant symbols pertaining to that particular example (in fact it delegates this step to Prolog).
It follows that the search complexity of \metagol{} is independent of the number of constant symbols and is the same\footnote{Metagol is sensitive to the size of the examples and to the computational complexity of a hypothesis because, as Schapire showed \cite{schapire:1990}, if checking whether a hypothesis $H$ covers an example $e$ cannot be performed in time polynomial in the size of $e$ and $H$ then $H$ cannot be learned in time polynomial in the size of $e$ and $H$, i.e. Metagol needs to execute a learned program on the example.} as Proposition \ref{prop:hs1}.


By contrast, \hex{} searches for a program by instantiating metarules in a bottom-up manner where the body atoms of metarules need to be grounded.
This approach means that the number of options that \hex{} considers is not only a function of the number of metarules and predicate symbols (as is the case for Metagol), but it is also a function of the number of constant symbols because it needs to ground the first-order variables in a metarule.
Even in the more efficient forward-chained MIL encoding, which incrementally imports new constants, body atoms are ground using many constant symbols unrelated to the examples.
Any constant that can be marked as a state will be used to ground atoms.
Therefore, the search complexity of \hex{} is bounded by $(mp^3c^6)^n$, where $m$ is the number of metarules, $p$ is the number of predicate symbols, $n$ is a maximum program size, and $c$ is the number of constant symbols.

For simplicity, the above complexity reasoning was for the first-order systems.
We can easily apply the same reasoning to the abstracted MIL setting.





\section{Experiments}
\label{sec:experiments}

Our main claim is that compared to learning first-order programs, learning higher-order programs can improve learning performance.
Theorem \ref{thm:ratio} supports this claim and shows that, compared to unabstracted MIL, abstraction in MIL reduces sample complexity proportional to the reduction in the number of clauses required to represent hypotheses.
We now experimentally\footnote{
All the experimental code and materials, including the code for Metagol, HEXMIL, and their higher-order extensions, is available at https://github.com/andrewcropper/mlj19-metaho
} explore this result.
We describe four experiments which compare the performance when learning first-order and higher-order programs.
We test the null hypotheses:

\begin{description}
\item [\textbf{Null hypothesis 1}] Learning higher-order programs cannot improve predictive accuracies
\item [\textbf{Null hypothesis 2}] Learning higher-order programs cannot reduce learning times
\end{description}

\noindent
To test these hypotheses we compare Metagol with \namea{} and HEXMIL with \nameb{}, i.e. we compare unabstracted MIL with abstracted MIL.

\subsection{Common materials}
In the Prolog experiments we use the same metarules and IBK in each experiment, i.e. the only variable in the Prolog experiments is the system (Metagol or \namea{}).
We use the metarules shown in Figure \ref{fig:bk}.
We use the higher-order definitions \tw{map/3}, \tw{until/4}, and \tw{ifthenelse/5} as IBK.
We run the Prolog experiments using SWI-Prolog 7.6.4 \cite{swipl}.

We tried to use the same experimental methodology in the ASP HEXMIL experiments as in the Prolog experiments but HEXMIL failed to learn any programs (first or higher-order) because of scalability issues.
Therefore, in each ASP experiment we use the exact metarules and background relations necessary to represent the target hypotheses.
We run the ASP experiments using Hexlite 1.0.0\footnote{https://github.com/hexhex/hexlite}.
We run Hexlite with the \emph{flpcheck} disabled.
We also set Hexlite to enumerate a single model.

\subsection{Robot waiter}

Imagine teaching a robot to pour tea and coffee at a dinner table, where each setting has an indication of whether the guest prefers tea or coffee.
Figure \ref{fig:waiter} shows an example in terms of initial and final states.
This experiment focuses on learning a general robot waiter strategy \cite{crop:metagolo} from a set of examples.

\begin{figure}[ht]
\centering
\begin{tabular} {c|c}
\includegraphics[scale=0.5]{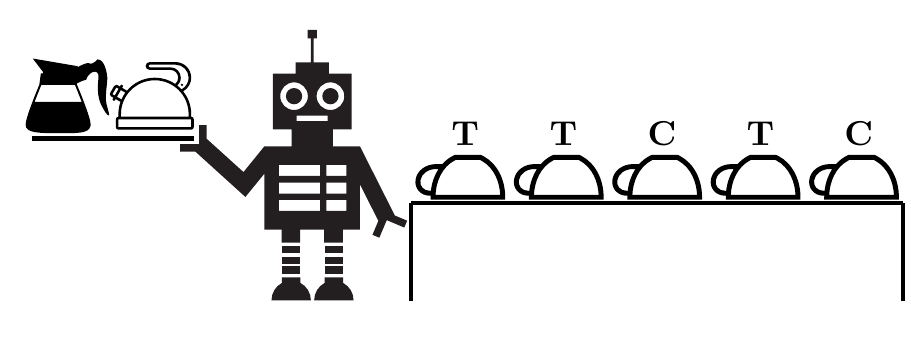} &
\includegraphics[scale=0.5]{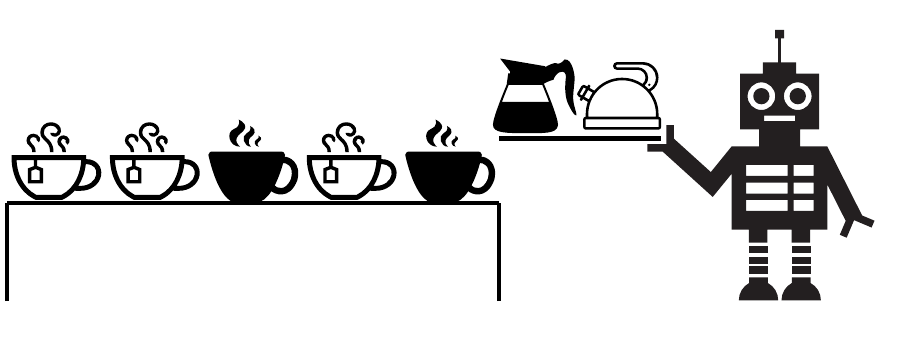} \\
(a) Initial state & (b) Final state
\end{tabular}
\caption{Figures (a) and (b) show initial/final state waiter examples respectively.
In the initial state, the cups are empty and each guest has a preference for tea \textbf{(T)} or coffee \textbf{(C)}.
In the final state, the cups are facing up and are full with the guest's preferred drink. }
\label{fig:waiter}
\end{figure}

\subsubsection{Materials}
Examples are \tw{f/2} atoms where the first argument is the initial state and the second is the final state.
A state is a list of ground Prolog atoms.
In the initial state, the robot starts at position 0, there are $d$ cups facing down at positions $0,\dots,d-1$; and for each cup there is a preference for tea or coffee.
In the final state, the robot is at position $d$; all the cups are facing up; and each cup is filled with the preferred drink.
We allow the robot to perform the fluents and actions (defined as compiled BK) shown in Figure \ref{fig:waiter-bk}.

\begin{figure}[ht]
\centering
\normalsize
\begin{tabular}{|ll|}
\hline
\tw{at\_end/1}& \tw{wants\_tea/1}\\
\tw{wants\_coffee/1}& \tw{move\_left/2}\\
\tw{move\_right/2}& \tw{turn\_cup\_over/2}\\
\tw{pour\_tea/2}& \tw{pour\_coffee/2}\\
\hline
\end{tabular}
\caption{
Compiled BK in the robot waiter experiment.
We omit the definitions for brevity.
}
\label{fig:waiter-bk}
\end{figure}

\noindent
We generate positive examples as follows.
For the Prolog experiments, for the initial state we select a random integer $d$ from the interval $[1,20]$ as the number of cups.
For the ASP experiments the interval is $[1,5]$.
For each cup, we randomly select whether the preferred drink is tea or coffee and set it facing down.
For the final state, we update the initial state so that each cup is facing up and is filled with the preferred drink.
To generate negative examples, we repeat the aforementioned procedure but we modify the final state so that the drink choice is incorrect for a random subset of $k>0$ drinks.

\subsubsection{Method}
Our experimental method is as follows.
For each learning system $s$ and for each $m$ in $\{1,2,\dots,10\}$:
\begin{enumerate}
    \item Generate $m$ positive and $m$ negative training examples
    \item Generate 1000 positive and 1000 negative testing example
    \item Use $s$ to learn a program $p$ using the training examples
    \item Measure the predictive accuracy of $p$ using the testing examples
\end{enumerate}

\noindent
If no program is found in 10 minutes then we deem that every testing example is false.
We measure mean predictive accuracies, mean learning times, and standard errors of the mean over 10 repetitions.

\subsubsection{Results}
Figure \ref{fig:prolog-waiter-res} shows that in all cases \metaho{} learns programs with higher predictive accuracies and lower learning times than Metagol.
Figure \ref{fig:asp-waiter-res} shows similar results when comparing \hex{} with \hexho{}.
We can explain these results by looking at example programs learned by Metagol and \metaho{} shown in Figures \ref{fig:waiter-res-fo} and \ref{fig:waiter-res-ho} respectively.
Although both programs are general and handle any number of guests and any assignment of drink preferences, the program learned by \metaho{} is smaller than the one learned by Metagol.
Whereas Metagol learns a recursive program, \metaho{} avoids recursion and uses the higher-order abstraction \tw{until/4}.
The abstraction \tw{until/4} essentially removes the need to learn a recursive two clause definition to move along the dinner table.
Likewise, \metaho{} uses the abstraction \tw{ifthenelse/5} to remove the need to learn two clauses to decide which drink to pour.
The compactness of the higher-order program affects predictive accuracies because, whereas \metaho{} almost always finds the target hypothesis in the allocated time, Metagol often struggles because the programs are too large, as explained by our theoretical results in Section \ref{sec:sc}.
The results from this experiment suggest that we can reject null hypotheses 1 and 2.

Although we are not directly comparing the Prolog and ASP implementations of MIL, it is interesting to note that despite having more irrelevant BK, more irrelevant metarules, and having larger training instances, \metaho{} outperforms \hexho{} in all cases, both in terms of predictive accuracies and learning times.
Figure \ref{fig:asp-waiter-res} also shows that both \hex{} and \hexho{} do not scale well in the number of training examples, especially the learning times.
Our results in Section \ref{sec:complex} help explain the poor scalability of \hex{} and \hexho{} because more training examples typically means more constant symbols which in turn means a larger search complexity for both \hex{} and \hexho{}, although this issue can be mitigated using state abstraction \cite{hexmil}.

\begin{figure}[ht]
\centering
\begin{tabular} {cc}
\begin{tikzpicture}[scale=0.70]
    \begin{axis}[
    ymin=0,
    ymax=100,
    xmin=2,
    xmax=20,
    xtick={0,2,4,...,20},
    ylabel absolute,
    ylabel style={yshift=-6mm},
    xlabel=No. training examples,
    ylabel=Accuracy (\%),
    legend style={legend pos=south east,font=\small},
    every axis legend/.append style={nodes={right}}]

\addplot+[gray,mark=square*,mark options={color=gray},error bars/.cd,y dir=both,y explicit] coordinates {
(2,61) +- (0,4)
(4,60) +- (0,5)
(6,79) +- (0,5)
(8,89) +- (0,4)
(10,88) +- (0,4)
(12,75) +- (0,8)
(14,88) +- (0,6)
(16,80) +- (0,8)
(18,88) +- (0,6)
(20,78) +- (0,7)
    };

\addplot+[blue,mark=*,mark options={fill=blue},error bars/.cd,y dir=both,y explicit] coordinates {
(2,87) +- (0,4)
(4,95) +- (0,1)
(6,95) +- (0,1)
(8,94) +- (0,1)
(10,92) +- (0,1)
(12,100) +- (0,0)
(14,98) +- (0,1)
(16,100) +- (0,0)
(18,98) +- (0,1)
(20,98) +- (0,1)
    };

\addplot+[mark=none,dotted] coordinates {
    (2,50)
    (20,50)
};

    \legend{Metagol,\metaho{},default}
    \end{axis}
  \end{tikzpicture} &
\begin{tikzpicture}[scale=0.70]
    \begin{axis}[
    ymin=0,
    xmin=2,
    xmax=20,
    xtick={0,2,4,...,20},
    xlabel=No. training examples,
    ylabel=Seconds,
    ylabel absolute,
    ylabel style={yshift=-4mm},
    legend style={legend pos=north west,font=\small},
    every axis legend/.append style={nodes={right}}]

\addplot+[gray,mark=square*,mark options={color=gray},error bars/.cd,y dir=both,y explicit] coordinates {
(2,85) +- (0,31)
(4,288) +- (0,90)
(6,239) +- (0,86)
(8,143) +- (0,70)
(10,111) +- (0,57)
(12,316) +- (0,94)
(14,165) +- (0,74)
(16,344) +- (0,75)
(18,267) +- (0,69)
(20,268) +- (0,90)
    };

\addplot+[blue,mark=*,mark options={fill=blue},error bars/.cd,y dir=both,y explicit] coordinates {
(2,14) +- (0,4)
(4,33) +- (0,11)
(6,16) +- (0,5)
(8,10) +- (0,4)
(10,22) +- (0,18)
(12,24) +- (0,9)
(14,6) +- (0,2)
(16,16) +- (0,5)
(18,16) +- (0,6)
(20,15) +- (0,6)
    };

    \legend{Metagol,\metaho{}}
    \end{axis}
  \end{tikzpicture} \\
(a) Predictive accuracies & (b) Learning times
\end{tabular}
\caption{Prolog robot waiter experiment results which show learning performance when varying the number of training examples. }
\label{fig:prolog-waiter-res}
\end{figure}
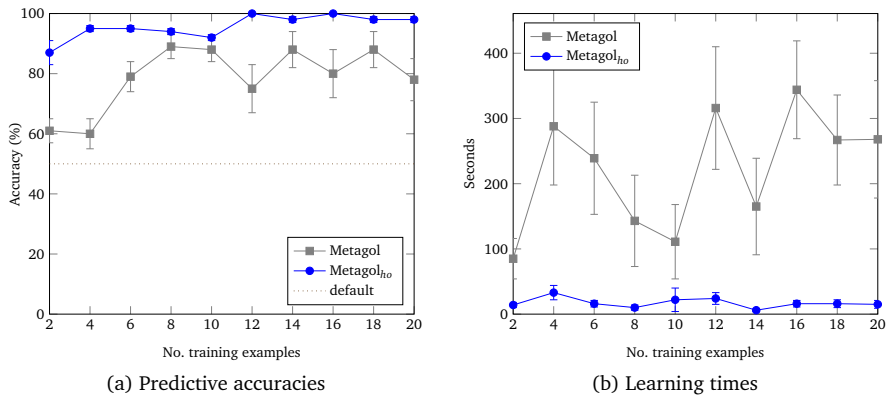

\begin{figure}[ht]
\centering
\begin{tabular} {cc}
\begin{tikzpicture}[scale=0.70]
    \begin{axis}[
    ymin=0,
    ymax=100,
    xmin=2,
    xmax=20,
    xtick={0,2,4,...,20},
    ylabel absolute,
    ylabel style={yshift=-6mm},
    xlabel=No. training examples,
    ylabel=Accuracy (\%),
    legend style={legend pos=south east,font=\small},
    every axis legend/.append style={nodes={right}}]

\addplot+[gray,mark=square*,mark options={color=gray},error bars/.cd,y dir=both,y explicit] coordinates {
(2,49) +- (0,0)
(4,51) +- (0,1)
(6,50) +- (0,0)
(8,50) +- (0,0)
(10,50) +- (0,0)
(12,50) +- (0,0)
(14,50) +- (0,0)
(16,50) +- (0,0)
(18,50) +- (0,0)
(20,50) +- (0,0)
    };

\addplot+[blue,mark=*,mark options={fill=blue},error bars/.cd,y dir=both,y explicit] coordinates {
(2,68) +- (0,4)
(4,77) +- (0,4)
(6,85) +- (0,1)
(8,90) +- (0,2)
(10,87) +- (0,2)
(12,90) +- (0,4)
(14,90) +- (0,4)
(16,83) +- (0,6)
(18,87) +- (0,6)
(20,87) +- (0,6)

    };

\addplot+[mark=none,dotted] coordinates {
    (2,50)
    (20,50)
};

    \legend{\hex{},\hexho{},default}
    \end{axis}
  \end{tikzpicture} &
\begin{tikzpicture}[scale=0.70]
    \begin{axis}[
    ymin=0,
    xmin=2,
    xmax=20,
    xtick={0,2,4,...,20},
    xlabel=No. training examples,
    ylabel=Seconds,
    ylabel absolute,
    ylabel style={yshift=-4mm},
    legend style={legend pos=south east,font=\small},
    every axis legend/.append style={nodes={right}}]

\addplot+[gray,mark=square*,mark options={color=gray},error bars/.cd,y dir=both,y explicit] coordinates {
(2,303) +- (0,98)
(4,424) +- (0,77)
(6,600) +- (0,0)
(8,600) +- (0,0)
(10,600) +- (0,0)
(12,600) +- (0,0)
(14,600) +- (0,0)
(16,600) +- (0,0)
(18,600) +- (0,0)
(20,600) +- (0,0)
    };

\addplot+[blue,mark=*,mark options={fill=blue},error bars/.cd,y dir=both,y explicit] coordinates {
(2,18) +- (0,5)
(4,29) +- (0,7)
(6,80) +- (0,9)
(8,114) +- (0,13)
(10,150) +- (0,43)
(12,272) +- (0,50)
(14,289) +- (0,58)
(16,328) +- (0,63)
(18,382) +- (0,48)
(20,398) +- (0,60)
    };

    \legend{\hex{},\hexho{}}
    \end{axis}
  \end{tikzpicture} \\
(a) Predictive accuracies & (b) Learning times
\end{tabular}
\caption{ASP robot waiter experiment results which show learning performance when varying the number of training examples. }
\label{fig:asp-waiter-res}
\end{figure}

\begin{figure}[ht]
\centering
\begin{tabular}{|c|}
\hline
\begin{lstlisting}
f(A,B):-turn_cup_over(A,C),f1(C,B).
f1(A,B):-move_right(A,B),at_end(B).
f1(A,B):-f2(A,C),f1(C,B).
f2(A,B):-wants_coffee(A),pour_coffee(A,B).
f2(A,B):-move_right(A,C),turn_cup_over(C,B).
f2(A,B):-wants_tea(A),pour_tea(A,B).
\end{lstlisting} \\
\hline
\end{tabular}
\caption{An example first-order waiter program learned by Metagol.}
\label{fig:waiter-res-fo}
\end{figure}

\begin{figure}[ht]
\centering
\begin{tabular}{|c|}
\hline
\begin{lstlisting}
f(A,B):-until(A,B,at_end,f1).
f1(A,B):-turn_cup_over(A,C),f2(C,B).
f2(A,B):-f3(A,C),move_right(C,B).
f3(A,B):-ifthenelse(A,B,wants_coffee,pour_coffee,pour_tea).
\end{lstlisting} \\
\hline
\end{tabular}
\caption{An example higher-order waiter program learned by \metaho{}.}
\label{fig:waiter-res-ho}
\end{figure}
\subsection{Chess strategy}
Programming chess strategies is a difficult task for humans \cite{bratko1980representation}.
For example, consider maintaining a wall of pawns to support promotion \cite{harris1988heuristic}.
In this case, we might start by trying to inductively program the simple situation in which a black pawn wall advances without interference from white.
Figure \ref{metafunc:fig:chess_fig} shows such an example, where in the initial state the pawns are at different ranks and in the final state all the pawns have advanced to rank 8 but the other pieces have remained in the initial positions.
In this experiment, we try to learn such strategies.

\begin{figure}[ht]
\centering
\begin{tabular} {cc}
\includegraphics[scale=0.4]{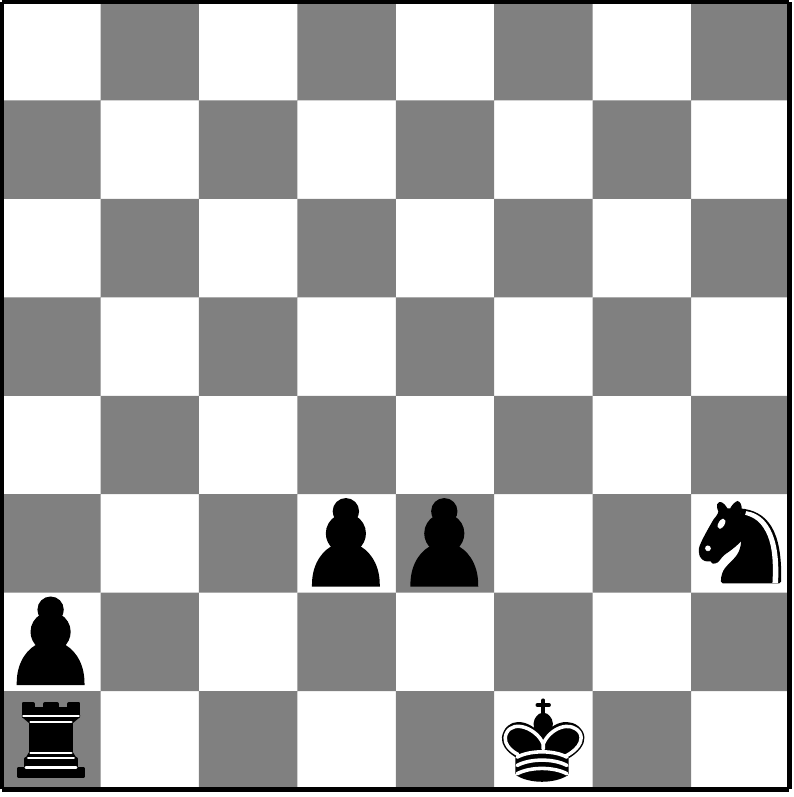} &
\includegraphics[scale=0.4]{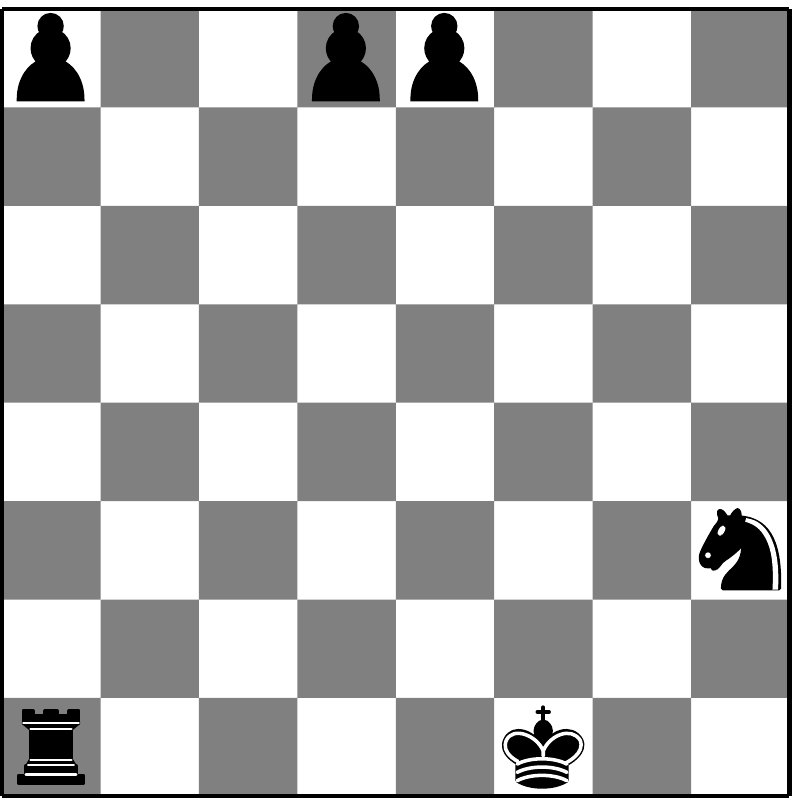} \\
(a) Initial state & (b) Final state
\end{tabular}
\caption{Chess initial/final state example.}
\label{metafunc:fig:chess_fig}
\end{figure}

\subsubsection{Materials}

Examples are \tw{f/2} atoms where the first argument is the initial state and the second is the final state.
A state is a list of pieces, where a piece is denoted as a tuple of the form \tw{(Type,Id,X,Y)}, where \tw{Type} is the type (king=k, pawn=p, etc.), \tw{Id} is a unique identifier, \tw{X} is the file, and \tw{Y} is the rank.
We generate a positive example as follows.
For the initial state for the Prolog experiments, we select a random subset of $n$ pieces from the interval $[1,16]$ and randomly place them on the board.
For the ASP experiments the interval is $[1,5]$. 
For the final state, we update the initial state so that each pawn finishes at rank 8.
To generate negative examples, we repeat the aforementioned procedure but we randomise the final state positions whilst ensuring that the input/output pair is not a positive example.
We use the compiled BK shown in Figure \ref{fig:chess-bk}.

\begin{figure}[ht]
\centering
\begin{tabular}{|c|}
\hline
\begin{lstlisting}
rank8((_,_,_,8)).
not_rank8(A):-
    \+rank8(A).
pawn((p,_,_,_)).
not_pawn(A):-
    \+pawn(A).
head([A|_],A).
tail([_|T],T).
empty([]).
hold(A,A).
forward((Type,Id,X,Y1),(Type,Id,X,Y2)):-
    Y1 < 8,
    Y2 is Y1+1.
\end{lstlisting} \\
\hline
\end{tabular}
\caption{Compiled BK used in the chess experiment.}
\label{fig:chess-bk}.
\end{figure}

\subsubsection{Method}
The experimental method is the same as in Experiment 1.

\subsubsection{Results}

Figure \ref{fig:prolog-chess-res} shows that in all cases \metaho{} learns programs with higher predictive accuracies and lower learning times than Metagol.
Figure \ref{fig:prolog-chess-res} shows that \metaho{} learns programs approaching 100\% accuracy after around six examples.
By contrast, Metagol learns programs with around default accuracy.
Figure \ref{fig:asp-chess-res} shows similar results when comparing \hex{} with \hexho{}.
The poor performance of Metagol and \hex{} is because they both rarely find solutions in the allocated time.
By contrast, \metaho{} and \hexho{} typically learn programs within two seconds.

We can again explain the performance discrepancies by looking at example learned programs in Figure \ref{fig:chess-progs}.
Figure \ref{fig:chess-ho} shows the compact higher-order program typically learned by \metaho{}.
This program is compact because it uses the abstractions \tw{map/3} and \tw{until/4}, where \tw{map/3} decomposes the problem into smaller sub-goals of moving a single piece to rank eight and \tw{until/4} solves the sub-problem of moving a pawn to rank eight.
These sub-goals are solved by the invented \tw{f1/2} predicate.
By contrast, Figure \ref{fig:chess-fo} shows the large target first-order program that Metagol struggled to learn.
As shown in Proposition \ref{prop:hs1}, the MIL hypothesis space grows exponentially in the size of the target hypothesis, which is why the larger first-order program is more difficult to learn.
The results from this experiment suggest that we can reject null hypotheses 1 and 2.


\begin{figure}[ht]
\centering
\begin{tabular} {cc}
\begin{tikzpicture}[scale=0.70]
    \begin{axis}[
    ymin=0,
    ymax=100,
    xmin=2,
    xmax=20,
    xtick={0,2,4,...,20},
    ylabel absolute,
    ylabel style={yshift=-6mm},
    xlabel=No. training examples,
    ylabel=Accuracy (\%),
    legend style={legend pos=south east,font=\small},
    every axis legend/.append style={nodes={right}}]

\addplot+[gray,mark=square*,mark options={color=gray},error bars/.cd,y dir=both,y explicit] coordinates {
(2,58) +- (0,3)
(4,58) +- (0,4)
(6,55) +- (0,3)
(8,59) +- (0,4)
(10,53) +- (0,3)
(12,53) +- (0,3)
(14,51) +- (0,1)
(16,50) +- (0,0)
(18,53) +- (0,3)
(20,50) +- (0,0)
    };

\addplot+[blue,mark=*,mark options={fill=blue},error bars/.cd,y dir=both,y explicit] coordinates {
(2,72) +- (0,7)
(4,89) +- (0,5)
(6,98) +- (0,0)
(8,98) +- (0,0)
(10,97) +- (0,0)
(12,98) +- (0,0)
(14,95) +- (0,3)
(16,98) +- (0,0)
(18,99) +- (0,0)
(20,99) +- (0,0)
    };

\addplot+[mark=none,dotted] coordinates {
    (2,50)
    (20,50)
};

    \legend{Metagol,\metaho{},default}
    \end{axis}
  \end{tikzpicture} &
\begin{tikzpicture}[scale=0.70]
    \begin{axis}[
    ymin=0,
    xmin=2,
    xmax=20,
    xtick={0,2,4,...,20},
    xlabel=No. training examples,
    ylabel=Seconds,
    ylabel absolute,
    ylabel style={yshift=-4mm},
    legend style={legend pos=north west,font=\small},
    every axis legend/.append style={nodes={right}}]

\addplot+[gray,mark=square*,mark options={color=gray},error bars/.cd,y dir=both,y explicit] coordinates {
(2,26) +- (0,25)
(4,94) +- (0,61)
(6,423) +- (0,89)
(8,373) +- (0,93)
(10,540) +- (0,59)
(12,540) +- (0,59)
(14,540) +- (0,59)
(16,600) +- (0,0)
(18,541) +- (0,58)
(20,600) +- (0,0)
    };

\addplot+[blue,mark=*,mark options={fill=blue},error bars/.cd,y dir=both,y explicit] coordinates {
(2,0) +- (0,0)
(4,0) +- (0,0)
(6,0) +- (0,0)
(8,0) +- (0,0)
(10,0) +- (0,0)
(12,0) +- (0,0)
(14,0) +- (0,0)
(16,0) +- (0,0)
(18,0) +- (0,0)
(20,0) +- (0,0)
    };

    \legend{Metagol,\metaho{}}
    \end{axis}
  \end{tikzpicture} \\
(a) Predictive accuracies & (b) Learning times
\end{tabular}
\caption{Prolog chess experimental results which show predictive accuracy when varying the number of training examples. Note that \metaho{} typically learns a program in under two seconds.}
\label{fig:prolog-chess-res}
\end{figure}
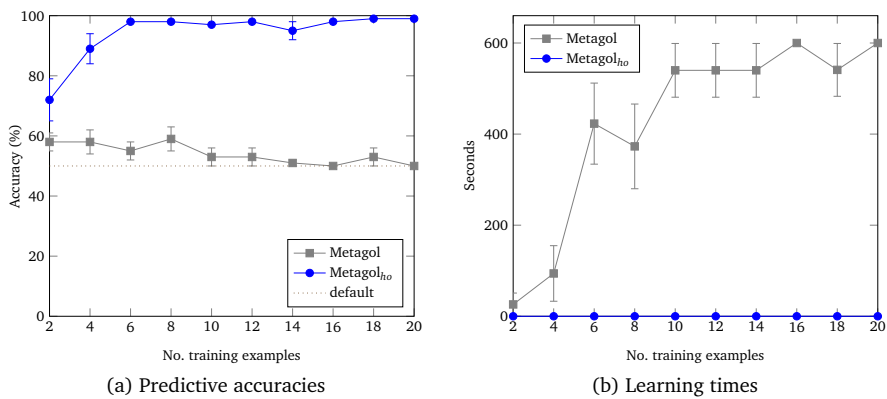

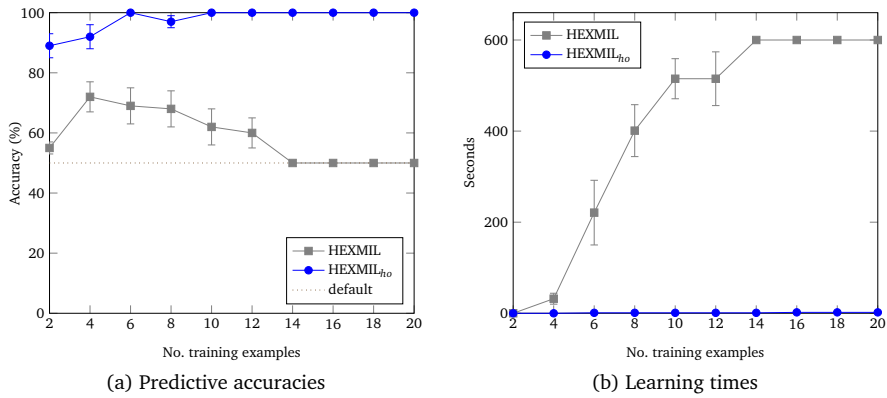
\begin{figure}[ht]
\centering
\begin{tabular} {cc}
\begin{tikzpicture}[scale=0.70]
    \begin{axis}[
    ymin=0,
    ymax=100,
    xmin=2,
    xmax=20,
    xtick={0,2,4,...,20},
    ylabel absolute,
    ylabel style={yshift=-6mm},
    xlabel=No. training examples,
    ylabel=Accuracy (\%),
    legend style={legend pos=south east,font=\small},
    every axis legend/.append style={nodes={right}}]

\addplot+[gray,mark=square*,mark options={color=gray},error bars/.cd,y dir=both,y explicit] coordinates {
(2,55) +- (0,2)
(4,72) +- (0,5)
(6,69) +- (0,6)
(8,68) +- (0,6)
(10,62) +- (0,6)
(12,60) +- (0,5)
(14,50) +- (0,0)
(16,50) +- (0,0)
(18,50) +- (0,0)
(20,50) +- (0,0)
    };

\addplot+[blue,mark=*,mark options={fill=blue},error bars/.cd,y dir=both,y explicit] coordinates {
(2,89) +- (0,4)
(4,92) +- (0,4)
(6,100) +- (0,0)
(8,97) +- (0,2)
(10,100) +- (0,0)
(12,100) +- (0,0)
(14,100) +- (0,0)
(16,100) +- (0,0)
(18,100) +- (0,0)
(20,100) +- (0,0)
    };

\addplot+[mark=none,dotted] coordinates {
    (2,50)
    (20,50)
};

    \legend{\hex{},\hexho{},default}
    \end{axis}
  \end{tikzpicture} &
\begin{tikzpicture}[scale=0.70]
    \begin{axis}[
    ymin=0,
    xmin=2,
    xmax=20,
    xtick={0,2,4,...,20},
    xlabel=No. training examples,
    ylabel=Seconds,
    ylabel absolute,
    ylabel style={yshift=-4mm},
    legend style={legend pos=north west,font=\small},
    every axis legend/.append style={nodes={right}}]

\addplot+[gray,mark=square*,mark options={color=gray},error bars/.cd,y dir=both,y explicit] coordinates {
(2,0) +- (0,0)
(4,32) +- (0,12)
(6,221) +- (0,71)
(8,401) +- (0,57)
(10,515) +- (0,44)
(12,515) +- (0,59)
(14,600) +- (0,0)
(16,600) +- (0,0)
(18,600) +- (0,0)
(20,600) +- (0,0)
    };

\addplot+[blue,mark=*,mark options={fill=blue},error bars/.cd,y dir=both,y explicit] coordinates {
(2,0) +- (0,0)
(4,0) +- (0,0)
(6,1) +- (0,0)
(8,1) +- (0,0)
(10,1) +- (0,0)
(12,1) +- (0,0)
(14,1) +- (0,0)
(16,2) +- (0,0)
(18,2) +- (0,0)
(20,2) +- (0,0)
    };

    \legend{\hex{},\hexho{}}
    \end{axis}
  \end{tikzpicture} \\
(a) Predictive accuracies & (b) Learning times
\end{tabular}
\caption{ASP chess experimental results which show predictive accuracy when varying the number of training examples. Note that \hexho{} typically learns a program in under two seconds.}
\label{fig:asp-chess-res}
\end{figure}

\begin{figure}[ht]
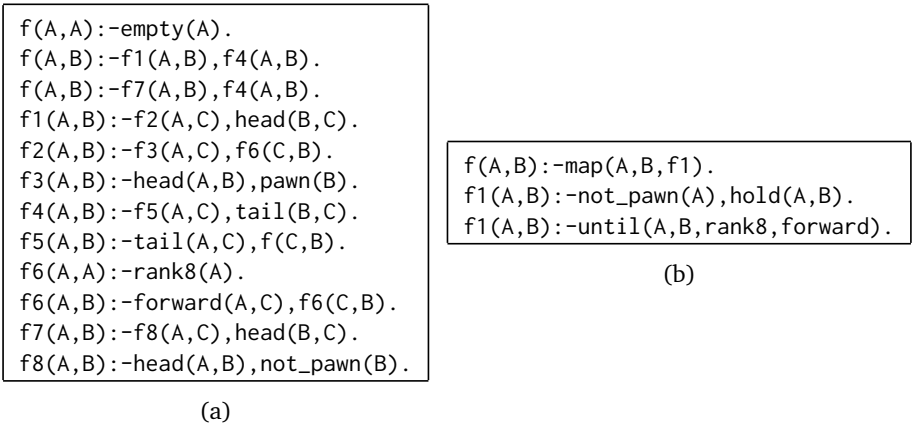

\normalsize
\begin{subfigure}[t]{.5\textwidth}
\centering
\normalsize
\begin{tabular}{|c|}
\hline
\begin{lstlisting}
f(A,A):-empty(A).
f(A,B):-f1(A,B),f4(A,B).
f(A,B):-f7(A,B),f4(A,B).
f1(A,B):-f2(A,C),head(B,C).
f2(A,B):-f3(A,C),f6(C,B).
f3(A,B):-head(A,B),pawn(B).
f4(A,B):-f5(A,C),tail(B,C).
f5(A,B):-tail(A,C),f(C,B).
f6(A,A):-rank8(A).
f6(A,B):-forward(A,C),f6(C,B).
f7(A,B):-f8(A,C),head(B,C).
f8(A,B):-head(A,B),not_pawn(B).
\end{lstlisting} \\
\hline
\end{tabular}
\normalsize
\caption{}
\label{fig:chess-fo}
\end{subfigure}%
\begin{subfigure}[t]{.5\textwidth}
\centering
\begin{tabular}{|c|}
\hline
\begin{lstlisting}
f(A,B):-map(A,B,f1).
f1(A,B):-not_pawn(A),hold(A,B).
f1(A,B):-until(A,B,rank8,forward).
\end{lstlisting} \\
\hline
\end{tabular}
\caption{}
\label{fig:chess-ho}
\end{subfigure}
\caption{Figure (a) shows the target first-order chess program, which Metagol could not learn within 10 minutes. Figure (b) shows the higher-order program often learned by \metaho{}. The higher-order program is clearly smaller than the first-order program, which is why \metaho{} could typically learn it within a couple of seconds.}
\label{fig:chess-progs}
\end{figure}
\subsection{Droplast}
\label{sec:droplast}

In this experiment, the goal is to learn a program that drops the last element from each sublist of a given list-of-lists -- a problem frequently used to evaluate program induction systems \cite{igor2}.
In this experiment, we try to learn a program that drops the last character from each string in a list of strings.
Figure \ref{fig:droplast-exs} shows input/output examples for this problem described using the \tw{f/2} predicate.

\begin{figure}[ht]
\centering
\normalsize
\begin{tabular}{|c|}
\hline
\begin{lstlisting}
f([alice,bob,carol],[alic,bo,caro]).
f([inductive,logic,programming],[inductiv,logi,programmin]]).
f([ferrara,orleans,london,kyoto],[ferrar,orlean,londo,kyot]).
\end{lstlisting} \\
\hline
\end{tabular}
\caption{
Examples of the droplast problem.
Note that in the experimental code we treat strings as lists of individual symbols, e.g. \tw{alice} is represented as \tw{[a,l,i,c,e]}.
}
\label{fig:droplast-exs}
\end{figure}

\subsubsection{Materials}
Examples are \tw{f/2} atoms where the first argument is the initial list and the second is the final list.
We generate positive examples as follows.
For the Prolog experiments, to form the input, we select a random integer $i$ from the interval $[1,10]$ as the number of sublists.
For each sublist $i$, we select a random integer $k$ from the interval $[1,10]$ and then sample with replacement a sequence of $k$ letters from the alphabet a-z to form the sublist $i$.
To form the output, we wrote a Prolog program to drop the last element from each sublist.
For the ASP experiments the interval for $i$ and $k$ is $[1,5]$.
We generate negative examples using a similar procedure, but instead of dropping the last element from each sublist,
we drop $j$ random elements (but not the last one) from each sublist, where $1 < j < k$.
We use the compiled BK shown in Figure \ref{fig:dl-bk}.


\begin{figure}[ht]
\centering
\normalsize
\begin{tabular}{|c|}
\hline
\begin{lstlisting}
head([H|_],H).
tail([_|T],T).
empty([]).
reverse(A,B):- ...
\end{lstlisting} \\
\hline
\end{tabular}
\caption{Compiled BK used in the droplast experiment.}
\label{fig:dl-bk}
\end{figure}

\subsubsection{Method}
The experimental method is the same as in Experiment 1.

\subsubsection{Results}
Figure \ref{fig:prolog-droplast-res} shows that \metaho{} achieved 100\% accuracy after two examples at which point it learned the program shown in Figure \ref{fig:dl-prog1}.
This program again uses abstractions to decompose the problem.
The predicate \tw{f/2} maps over the input list and applies \tw{f1/2} to each sublist to form the output list, thus abstracting away the reasoning for iterating over a list.
The invented predicate \tw{f1/2} drops the last element from a single list by reversing the list, calling \tw{tail/2} to drop the head element, and then reversing the shortened list back to the original order.
By contrast, Metagol was unable to learn any solutions because the corresponding first-order program is too long and the search is impractical, similar to the issues in the chess experiment.

Figure \ref{fig:asp-droplast-res} shows slightly unexpected results for the ASP experiment.
The figure shows that \hexho{} learns programs with higher predictive accuracies than \hex{} when given up to 14 training examples.
However, the predictive accuracies of \hexho{} progressively decreases given more examples.
This performance degradation is because, as we have already explained, \hex{} and \hexho{} do not scale well given more examples.
This inability to scale given more examples is clearly shown in Figure \ref{fig:asp-droplast-res}, which shows that the learning times of \hexho{} increase significantly given more training examples.


We repeated the \emph{droplast} experiment but replaced \tw{reverse/2} in the BK with the higher-order definition \tw{reduceback/3} and the compiled clause \tw{concat/3}.
In this scenario, \metaho{} learned the higher-order program shown in Figure \ref{fig:dl-prog2}.
This program now includes the invented predicate \tw{f3/2} which reverses a given list and is used twice in the program.
This more complex program highlights invention through the repeated calls to \tw{f3/2} and abstraction through the use of higher-order functions.

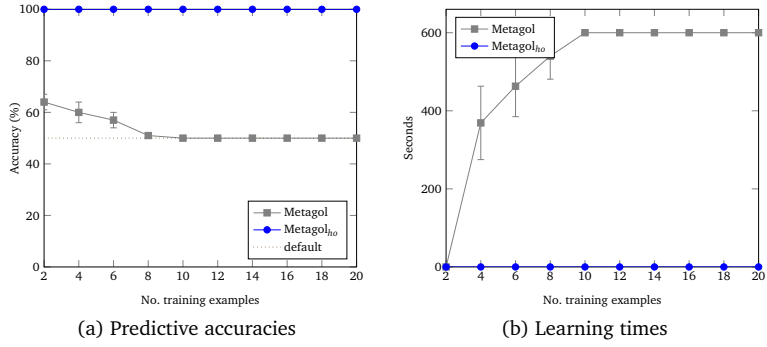
\begin{figure}[ht]
\centering
\begin{tabular} {cc}
\begin{tikzpicture}[scale=0.60]
    \begin{axis}[
    ymin=0,
    ymax=100,
    xmin=2,
    xmax=20,
    xtick={0,2,4,...,20},
    ylabel absolute,
    ylabel style={yshift=-6mm},
    xlabel=No. training examples,
    ylabel=Accuracy (\%),
    legend style={legend pos=south east,font=\small},
    every axis legend/.append style={nodes={right}}]

\addplot+[gray,mark=square*,mark options={color=gray},error bars/.cd,y dir=both,y explicit] coordinates {
(2,64) +- (0,3)
(4,60) +- (0,4)
(6,57) +- (0,3)
(8,51) +- (0,1)
(10,50) +- (0,0)
(12,50) +- (0,0)
(14,50) +- (0,0)
(16,50) +- (0,0)
(18,50) +- (0,0)
(20,50) +- (0,0)
    };

\addplot+[blue,mark=*,mark options={fill=blue},error bars/.cd,y dir=both,y explicit] coordinates {
(2,100) +- (0,0)
(4,100) +- (0,0)
(6,100) +- (0,0)
(8,100) +- (0,0)
(10,100) +- (0,0)
(12,100) +- (0,0)
(14,100) +- (0,0)
(16,100) +- (0,0)
(18,100) +- (0,0)
(20,100) +- (0,0)
    };

\addplot+[mark=none,dotted] coordinates {
    (2,50)
    (20,50)
};

    \legend{Metagol,\metaho{},default}
    \end{axis}
  \end{tikzpicture} &
\begin{tikzpicture}[scale=0.60]
    \begin{axis}[
    ymin=0,
    xmin=2,
    xmax=20,
    xtick={0,2,4,...,20},
    xlabel=No. training examples,
    ylabel=Seconds,
    ylabel absolute,
    ylabel style={yshift=-4mm},
    legend style={legend pos=north west,font=\small},
    every axis legend/.append style={nodes={right}}]

\addplot+[gray,mark=square*,mark options={color=gray},error bars/.cd,y dir=both,y explicit] coordinates {
(2,0) +- (0,0)
(4,369) +- (0,94)
(6,463) +- (0,78)
(8,540) +- (0,59)
(10,600) +- (0,0)
(12,600) +- (0,0)
(14,600) +- (0,0)
(16,600) +- (0,0)
(18,600) +- (0,0)
(20,600) +- (0,0)
    };

\addplot+[blue,mark=*,mark options={fill=blue},error bars/.cd,y dir=both,y explicit] coordinates {
(2,0) +- (0,0)
(4,0) +- (0,0)
(6,0) +- (0,0)
(8,0) +- (0,0)
(10,0) +- (0,0)
(12,0) +- (0,0)
(14,0) +- (0,0)
(16,0) +- (0,0)
(18,0) +- (0,0)
(20,0) +- (0,0)
    };

    \legend{Metagol,\metaho{}}
    \end{axis}
  \end{tikzpicture} \\
(a) Predictive accuracies & (b) Learning times
\end{tabular}
\caption{Prolog droplast experimental results which show predictive accuracy when varying the number of training examples. Note that \metaho{} typically learns a program in under two seconds.}
\label{fig:prolog-droplast-res}
\end{figure}

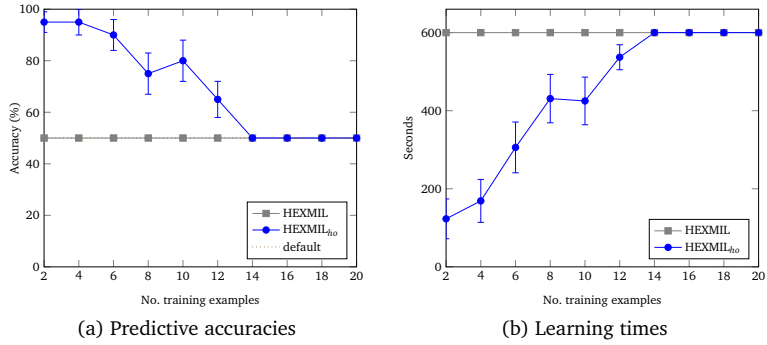
\begin{figure}[ht]
\centering
\begin{tabular} {cc}
\begin{tikzpicture}[scale=0.60]
    \begin{axis}[
    ymin=0,
    ymax=100,
    xmin=2,
    xmax=20,
    xtick={0,2,4,...,20},
    ylabel absolute,
    ylabel style={yshift=-6mm},
    xlabel=No. training examples,
    ylabel=Accuracy (\%),
    legend style={legend pos=south east,font=\small},
    every axis legend/.append style={nodes={right}}]

\addplot+[gray,mark=square*,mark options={color=gray},error bars/.cd,y dir=both,y explicit] coordinates {
(2,50) +- (0,0)
(4,50) +- (0,0)
(6,50) +- (0,0)
(8,50) +- (0,0)
(10,50) +- (0,0)
(12,50) +- (0,0)
(14,50) +- (0,0)
(16,50) +- (0,0)
(18,50) +- (0,0)
(20,50) +- (0,0)
    };

\addplot+[blue,mark=*,mark options={fill=blue},error bars/.cd,y dir=both,y explicit] coordinates {
(2,95) +- (0,4)
(4,95) +- (0,5)
(6,90) +- (0,6)
(8,75) +- (0,8)
(10,80) +- (0,8)
(12,65) +- (0,7)
(14,50) +- (0,0)
(16,50) +- (0,0)
(18,50) +- (0,0)
(20,50) +- (0,0)
    };

\addplot+[mark=none,dotted] coordinates {
    (2,50)
    (20,50)
};

    \legend{\hex{},\hexho{},default}
    \end{axis}
  \end{tikzpicture} &
\begin{tikzpicture}[scale=0.60]
    \begin{axis}[
    ymin=0,
    xmin=2,
    xmax=20,
    xtick={0,2,4,...,20},
    xlabel=No. training examples,
    ylabel=Seconds,
    ylabel absolute,
    ylabel style={yshift=-4mm},
    legend style={legend pos=south east,font=\small},
    every axis legend/.append style={nodes={right}}]

\addplot+[gray,mark=square*,mark options={color=gray},error bars/.cd,y dir=both,y explicit] coordinates {
(2,600) +- (0,0)
(4,600) +- (0,0)
(6,600) +- (0,0)
(8,600) +- (0,0)
(10,600) +- (0,0)
(12,600) +- (0,0)
(14,600) +- (0,0)
(16,600) +- (0,0)
(18,600) +- (0,0)
(20,600) +- (0,0)
    };

\addplot+[blue,mark=*,mark options={fill=blue},error bars/.cd,y dir=both,y explicit] coordinates {
(2,123) +- (0,51)
(4,169) +- (0,55)
(6,306) +- (0,65)
(8,431) +- (0,62)
(10,425) +- (0,61)
(12,537) +- (0,32)
(14,600) +- (0,0)
(16,600) +- (0,0)
(18,600) +- (0,0)
(20,600) +- (0,0)
    };

    \legend{\hex{},\hexho{}}
    \end{axis}
  \end{tikzpicture} \\
(a) Predictive accuracies & (b) Learning times
\end{tabular}
\caption{ASP droplast experimental results which show predictive accuracy when varying the number of training examples.}
\label{fig:asp-droplast-res}
\end{figure}

\begin{figure}[ht]
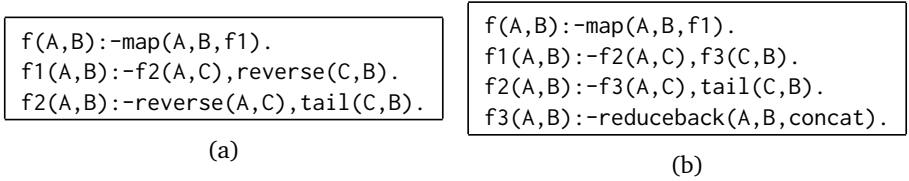

\normalsize
\begin{subfigure}[t]{.5\textwidth}
\centering
\normalsize
\begin{tabular}{|c|}
\hline
\begin{lstlisting}
f(A,B):-map(A,B,f1).
f1(A,B):-f2(A,C),reverse(C,B).
f2(A,B):-reverse(A,C),tail(C,B).
\end{lstlisting} \\
\hline
\end{tabular}
\normalsize
\caption{}
\label{fig:dl-prog1}
\end{subfigure}%
\begin{subfigure}[t]{.5\textwidth}
\centering
\begin{tabular}{|c|}
\hline
\begin{lstlisting}
f(A,B):-map(A,B,f1).
f1(A,B):-f2(A,C),f3(C,B).
f2(A,B):-f3(A,C),tail(C,B).
f3(A,B):-reduceback(A,B,concat).
\end{lstlisting} \\
\hline
\end{tabular}
\caption{}
\label{fig:dl-prog2}
\end{subfigure}
\caption{Figure (a) shows the higher-order program often learned by \metaho{}.
Figure (b) shows a more complex program learned by \metaho{} when we repeated the experiment but disallowed \metaho{} to use \tw{reverse/2} and instead gave it \tw{reduceback/3} and \tw{concat/3}.}
\label{fig:dl-progs}
\end{figure}

\subsubsection{Further discussion}

To further demonstrate invention and abstraction, consider learning a \emph{double droplast} program which extends the droplast problem so that, in addition to dropping the last element from each sublist, it also drops the last sublist.
Figure \ref{fig:ddl-exs} shows examples of this problem, again represented as the target predicate \tw{f/2}.
Given two examples of this problem, \metaho{} learns the program shown in Figure \ref{fig:ddl1}.
For readability Figure \ref{fig:ddl2} shows the folded program where non-reused invented predicates are removed.
This program is similar to the program shown in Figure \ref{fig:dl-prog2} but it makes an additional final call to the invented predicate \tw{f1/2} which is used twice in the program, once as a higher-order argument in \tw{map/3} and again as a first-order predicate.
This form of higher-order abstraction and invention goes beyond anything in the existing literature.

\begin{figure}[ht]
\normalsize
\centering
\begin{tabular}{|c|}
\hline
\begin{lstlisting}
f([alice,bob,carol],[alic,bo]).
f([inductive,logic,programming],[inductiv,logi]]).
f([ferrara,orleans,london,kyoto],[ferrar,orlean,londo]).
\end{lstlisting} \\
\hline
\end{tabular}
\caption{Examples of the more complex double droplast problem.}
\label{fig:ddl-exs}
\end{figure}

\begin{figure}[ht]
\normalsize
\begin{subfigure}[t]{.5\textwidth}
\centering
\normalsize
\begin{tabular}{|c|}
\hline
\begin{lstlisting}
f(A,B):-f1(A,C),f2(C,B).
f1(A,B):-map(A,B,f2).
f2(A,B):-f3(A,C),f4(C,B).
f3(A,B):-f4(A,C),tail(C,B).
f4(A,B):-reduceback(A,B,concat).
\end{lstlisting} \\
\hline
\end{tabular}
\normalsize
\caption{}
\label{fig:ddl1}
\end{subfigure}%
\begin{subfigure}[t]{.5\textwidth}
\centering
\begin{tabular}{|c|}
\hline
\begin{lstlisting}
f(A,B):-map(A,C,f1),f1(C,B).
f1(A,B):-f2(A,C),tail(C,D),f2(D,B).
f2(A,B):-reduceback(A,B,concat).
\end{lstlisting} \\
\hline
\end{tabular}
\caption{}
\label{fig:ddl2}
\end{subfigure}
\caption{Figure (a) shows a the higher-order \emph{double droplast} program learned by \metaho{}.
For readability Figure (b) shows the folded program in which non-reused invented predicates are removed.
Note how in Figure (b) the predicate symbol \tw{f1/2} is used both as an argument to \tw{map/3} and as a standard literal in the clause defined by the head \tw{f(A,B)}.
}
\end{figure}
\subsection{Encryption}
\label{sec:encryption}

In this final experiment, we revisit the encryption example from the introduction.

\subsubsection{Materials}
Examples are \tw{f/2} atoms where the first argument is the encrypted string and the second is the unencrypted string.
For simplicity we only allow the letters a-z.
We generate a positive example as follows.
For the Prolog experiments we select a random integer $k$ from the interval $[1,20]$ to denote the unencrypted string length.
For the ASP experiments we select $k$ from the interval $[1,5]$. 
We sample with replacement a sequence $y$ of length $k$ from the set $\{a,b,\dots,z\}$.
The sequence $y$ denotes the unencrypted string.
We form the encrypted string $x$ by shifting each character in $y$ two places to the right, e.g. $a\mapsto c, b\mapsto d, \dots, z \mapsto b$.
The atom $f(x,y)$ thus represents a positive example.
To generate negative examples we repeat the aforementioned procedure but we shift each character by $n$ places where $0 \leq n < 25$ and $n \neq 2$.
For the BK we use the relations \tw{char\_to\_int/2}, \tw{int\_to\_char/2}, \tw{succ/2}, and \tw{prec/2}, where, for simplicity, \tw{succ(25,0)} and \tw{prec(0,25)} hold.

\subsubsection{Method}
The experimental method is the same as in Experiment 1.

\subsubsection{Results}
Figure \ref{fig:prolog-encryption-res} shows that, as with the other experiments, \namea{} learns programs with higher predictive accuracies and lower learning times than Metagol.
These results are as expected because, as shown in Figure \ref{fig:m22fo}, to represent the target encryption hypothesis as a first-order program \M{2}{2} requires seven clauses.
By contrast, as shown in Figure \ref{fig:m22ho}, to represent the target hypothesis as a higher-order program in \M{2}{2} requires three clauses with one additional higher-order variable in the \tw{map/3} abstraction.

We attempted to run the experiment using \hex{} and \hexho{}.
However, both systems failed to find any programs within the timelimit.
In fact, even in an extremely simple version of the experiment (where the alphabet contained only 10 letters, each string had at most 3 letters, and the character shift was +1) both systems failed to learn anything in the allocated time.
Our theoretical results in Section \ref{sec:complex} explain these empirical results.
In this scenario, the number of ways that the BK predicates can be chained together and instantiated is no longer tractable for \hex{}.
The experiment suggests that \hex{} needs to be better at determining which groundings are relevant to consistent hypotheses.

\begin{figure}[ht]
\centering
\begin{tabular} {cc}
\begin{tikzpicture}[scale=0.65]
    \begin{axis}[
    ymin=0,
    ymax=100,
    xmin=2,
    xmax=20,
    xtick={0,2,4,...,20},
    ylabel absolute,
    ylabel style={yshift=-6mm},
    xlabel=No. training examples,
    ylabel=Accuracy (\%),
    legend style={legend pos=south east,font=\small},
    every axis legend/.append style={nodes={right}}]

\addplot+[gray,mark=square*,mark options={color=gray},error bars/.cd,y dir=both,y explicit] coordinates {
(2,50) +- (0,0)
(4,50) +- (0,0)
(6,50) +- (0,0)
(8,50) +- (0,0)
(10,50) +- (0,0)
(12,50) +- (0,0)
(14,50) +- (0,0)
(16,50) +- (0,0)
(18,50) +- (0,0)
(20,50) +- (0,0)
    };

\addplot+[blue,mark=*,mark options={fill=blue},error bars/.cd,y dir=both,y explicit] coordinates {
(2,75) +- (0,8)
(4,90) +- (0,6)
(6,100) +- (0,0)
(8,100) +- (0,0)
(10,100) +- (0,0)
(12,100) +- (0,0)
(14,100) +- (0,0)
(16,100) +- (0,0)
(18,100) +- (0,0)
(20,100) +- (0,0)
    };

\addplot+[mark=none,dotted] coordinates {
    (2,50)
    (20,50)
};

    \legend{Metagol,\namea{},default}
    \end{axis}
  \end{tikzpicture} &
\begin{tikzpicture}[scale=0.65]
    \begin{axis}[
    ymin=0,
    xmin=2,
    xmax=20,
    xtick={0,2,4,...,20},
    xlabel=No. training examples,
    ylabel=Seconds,
    ylabel absolute,
    ylabel style={yshift=-4mm},
    legend style={legend pos=north west,font=\small},
    every axis legend/.append style={nodes={right}}]

\addplot+[gray,mark=square*,mark options={color=gray},error bars/.cd,y dir=both,y explicit] coordinates {
(2,180) +- (0,91)
(4,420) +- (0,91)
(6,600) +- (0,0)
(8,600) +- (0,0)
(10,600) +- (0,0)
(12,600) +- (0,0)
(14,600) +- (0,0)
(16,600) +- (0,0)
(18,600) +- (0,0)
(20,600) +- (0,0)
    };

\addplot+[blue,mark=*,mark options={fill=blue},error bars/.cd,y dir=both,y explicit] coordinates {
(2,18) +- (0,10)
(4,43) +- (0,26)
(6,6) +- (0,3)
(8,32) +- (0,23)
(10,14) +- (0,10)
(12,13) +- (0,12)
(14,20) +- (0,9)
(16,19) +- (0,13)
(18,5) +- (0,3)
(20,7) +- (0,5)
    };

    \legend{Metagol,\namea{}}
    \end{axis}
  \end{tikzpicture} \\
(a) Predictive accuracies & (b) Learning times
\end{tabular}
\caption{Prolog encryption experiment results which show learning performance when varying the number of training examples. }
\label{fig:prolog-encryption-res}
\end{figure}
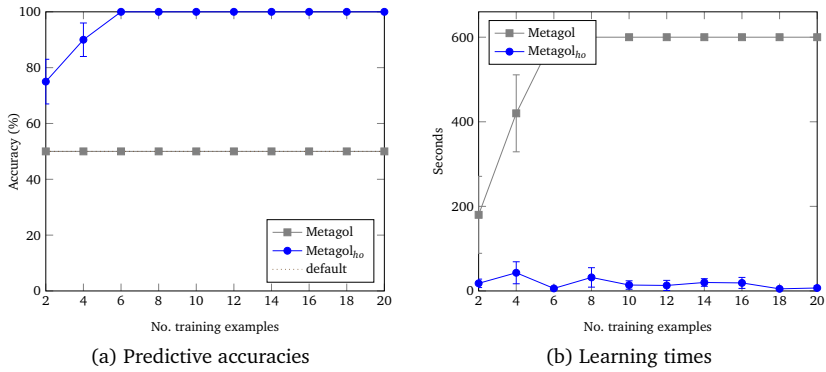

\subsection{Discussion}

Our main claim is that compared to learning first-order programs, learning higher-order programs can improve learning performance. Our experiments support this claim and show that learning higher-order programs can significantly improve predictive accuracies and reduce learning times.

Although it was not our purpose, our experiments also implicitly (implicitly because we do not directly compare the systems) show that Metagol outperforms HEXMIL, and similarly \namea{} outperforms \nameb{}.
Our empirical results contradict those by Kaminski et al. \cite{hexmil}, but support those by Morel et al. \cite{crop:typed}.
There are multiple explanations for this discrepancy.
We think that the main problem is the ASP grounding problem faced by HEXMIL: in most of our experiments, HEXMIL timed out during the grounding (and not solving) stage.
To alleviate this issue, future work could consider using state abstraction \cite{hexmil} to mitigate the grounding issues.

Also, by adjusting the experimental methodology, some of the results may change.
For instance, Kaminski et al. showed that HEXMIL can sometimes learn solutions quicker than Metagol because of conflict propagation in ASP.
They claim that this performance improvement is because Metagol only considers negative examples after inducing a program from the positive examples (as described in Section \ref{sec:metagol}).
Therefore, HEXMIL should benefit from more negative examples, but may suffer from fewer.

To summarise, although our empirical results suggest that Metagol outperforms HEXMIL, future work should more rigorously compare the two approaches on multiple domains along multiple dimensions (e.g. varying the numbers of examples, size of BK, etc.).
\section{Conclusions and further work}
\label{conc}

We have extended MIL to support learning higher-order programs by allowing for higher-order definitions to be included as background knowledge.
We showed that learning higher-order programs can reduce the textual complexity required to express target classes of programs which in turn reduces the hypothesis space.
Our sample complexity results show that learning higher-order programs can reduce the number of examples required to reach high predictive accuracies.
To learn higher-order programs, we introduced \metaho{}, a MIL learner which also supports higher-order predicate invention, such as inventing predicates for the higher-order abstractions \tw{map/3} and \tw{until/4}.
We also introduced \hexho{}, an ASP implementation of MIL that also supports learning higher-order programs.
Our experiments showed that, compared to learning first-order programs, learning higher-order programs can significantly improve predictive accuracies and reduce learning times.

\subsection{Limitations and future work}

\subsubsection{Metarules}

There are at least two limitations with our work regarding the choice of metarules when learning higher-order programs.

One issue is deciding which metarules to use.
Figure \ref{fig:bk} shows the 11 metarules used in our experiments.
Of these metarules, eight of them (the ones with only monadic or dyadic literals) are a subset of a \emph{derivationally} irreducible set of monadic and dyadic metarules \cite{crop:dreduce}.
We can therefore justify their selection because they are sufficient to learn any program in a slightly restricted subset of Datalog.
However, we have additionally used three \emph{curry} metarules with arities three, four, and five, which were not considered in the work on identifying derivationally irreducible metarules.
In addition, the curry metarules also include existentially quantified predicate arguments (e.g. $R$ in $P(A,B) \leftarrow Q(A,B,R)$).
Although these metarules seem intuitive and sensible to use, we have no theoretical justification for using them.
Future work should address this issue, such as by extending the existing work \cite{crop:dreduce} to include such metarules.

A second issue regarding the curry metarules is that when used with abstractions they each require an extra clause in the learned program.
Our motivation for learning higher-order programs was to reduce the number of clauses necessary to express a target theory.
Although our theoretical and experimental results support this claim, further improvements can be made.
For instance, suppose you are given examples of the concept $f(x,y)$ where $x$ is a list of integers and $y$ is $x$ but reversed, where each element has had one added to it, and then doubled, such as  $f([1,2,3],[8,6,4])$. Then \metaho{} could learn the following program given the metarules used in Figure  \ref{fig:bk}:

\begin{center}
\begin{minipage}{.43\textwidth}
\centering
\begin{lstlisting}[frame=single]
f(A,B):-f1(A,C),f5(C,B).
f1(A,B):-f3(A,C),f4(C,B).
f3(A,B):-reduceback(A,B,concat).
f4(A,B):-map(A,B,succ).
f5(A,B):-map(A,B,double).
\end{lstlisting}
\end{minipage}
\end{center}

\noindent
This program requires five clauses.
By contrast, a more compact representation would be:

\begin{center}
\begin{minipage}{.43\textwidth}
\centering
\begin{lstlisting}[frame=single]
f(A,B):-reduceback(A,C,concat),
    map(C,D,succ),
    map(D,B,double).
\end{lstlisting}
\end{minipage}
\end{center}

\noindent
This more compact program is formed of a single clause and four literals, so should therefore be easier to learn.
Future work should try to address this limitation of the current approach\footnote{This limitation is not specific to when learning higher-order programs. Curry metarules can also be used when learning first-order programs, where existentially quantified argument variables could be bound to constant symbols, rather than predicate symbols. In other words, this issue is a limitation of MIL but manifests itself clearly when learning higher-order programs.}.

\subsubsection{Higher-order definitions}
Our experiments rely on a few higher-order definitions, mostly based on higher-order programming concepts, such as \tw{map/3} and \tw{until/4}.
Future work should consider other higher-order concepts.
For instance, consider learning regular grammars, such as $a^*b^*c^*$.
To improve learning efficiency it would be desirable to encode the concept of \emph{Kleene star operator}\footnote{The Kleene star operator represents zero-or-more repetitions (here applications) of its argument.} as a higher-order definition, such as:

\begin{center}
\begin{minipage}{.25\textwidth}
\centering
\begin{lstlisting}[frame=single]
kstar(P,A,A).
kstar(P,A,B):-
    call(P,A,C),
    kstar(P,C,B).
\end{lstlisting}
\end{minipage}
\end{center}

\noindent
Similarly, we have used abstracted MIL to invent functional constructs.
Future work could consider inventing relational constructs.
For instance, consider this higher-order definition of a closure:

\begin{center}
\begin{tabular}{l}
closure(P,A,B) $\leftarrow$ P(A,B)\\
closure(P,A,B) $\leftarrow$ P(A,C), closure(P,C,B)
\end{tabular}
\end{center}

\noindent
We could use this definition to learn compact abstractions of relations, such as:

\begin{center}
\begin{tabular}{l}
ancestor(A,B) $\leftarrow$ closure(parent,A,B)\\
lessthan(A,B) $\leftarrow$ closure(increment,A,B)\\
subterm(A,B) $\leftarrow$ closure(headortail,A,B)
\end{tabular}
\end{center}



\subsubsection{Learning higher-order abstractions}

One clear limitation of the current approach is that we require user-provided higher-order definitions, such as \tw{map/3}.
In future work we want to learn or invent such definitions.
For instance, when learning a solution to the decryption program in the introduction it may be beneficial to learn and invent a sub-definition that corresponds to \tw{map/3}.
The program below shows such a scenario, where the definition \tw{decrypt1/3} corresponds to \tw{map/3}.

\begin{center}
\begin{minipage}{.37\textwidth}
\centering
\begin{lstlisting}[frame=single]
decrypt(A,B):-
    decrypt1(A,B,decrypt2).
decrypt1(A,B,C):-
    empty(A),
    empty(B).
decrypt1(A,B,C):-
    head(A,D),
    call(C,D,E),
    head(B,E),
    tail(A,F),
    tail(B,G),
    decrypt1(F,G,C).
decrypt2(A,B):-
    char_to_int(A,C),
    prec(C,D),
    int_to_char(D,B).
\end{lstlisting}
\end{minipage}
\end{center}
\noindent
Our preliminary work suggests that learning such definitions is possible.


\subsection{Summary}
In summary, our primary contribution is a demonstration of the value of higher-order abstractions and inventions in MIL.
We have shown that the techniques allow us to learn substantially more complex programs using fewer examples with less search.

\begin{acknowledgements}
We thank Stassa Patsantzis and Tobias Kaminski for helpful feedback on the paper.
\end{acknowledgements}

\bibliographystyle{plain}
\bibliography{manuscript}
\end{document}